\documentclass[conference,compsoc]{IEEEtran}\IEEEoverridecommandlockouts
\usepackage[caption=false,font=footnotesize,labelfont=sf,textfont=sf]{subfig}

\usepackage{hyperref}

\usepackage{amssymb}
\usepackage{mathtools}
\usepackage{amsthm}
\usepackage{graphicx}
\usepackage{booktabs}
\usepackage{algorithm}
\usepackage{algorithmic}

\usepackage{diagbox}
\usepackage{multirow}
\usepackage[normalem]{ulem}
\useunder{\uline}{\ul}{}

\usepackage{amsfonts}
\usepackage{amsmath}

\usepackage{amsthm}

\newcommand{\bfit}[1]{\textbf{\textit{#1}}}

\ifCLASSOPTIONcompsoc
  \usepackage[nocompress]{cite}
\else
  \usepackage{cite}
\fi

\begin{document}
%
\title{Neural Networks with (Low-Precision) Polynomial Approximations:\\New Insights and Techniques for Accuracy Improvement}


\author{
\IEEEauthorblockN{Chi Zhang\IEEEauthorrefmark{1},  
Jingjing Fan\IEEEauthorrefmark{1}, 
Man Ho Au \IEEEauthorrefmark{2}\thanks{* Correspondence
to: Man Ho Au $<$mhaau@polyu.edu.hk$>$}, 
Siu Ming Yiu\IEEEauthorrefmark{1}} 
  
\IEEEauthorblockA{\IEEEauthorrefmark{1}Department of Computer Science, The University of Hong Kong, Hong Kong, China} 
\IEEEauthorblockA{\IEEEauthorrefmark{2}Department of Computing, The Hong Kong Polytechnic University, Hong Kong, China} 

}

\maketitle

\begin{abstract}
Replacing non-polynomial functions (e.g., non-linear activation functions such as ReLU) in a neural network with their polynomial approximations is a standard practice in privacy-preserving machine learning. The resulting neural network, called polynomial approximation of neural network (PANN) in this paper, is compatible with advanced cryptosystems to enable privacy-preserving model inference. Using ``highly precise'' approximation, state-of-the-art PANN offers similar inference accuracy as the underlying backbone model. However, little is known about the effect of approximation, and existing literature often determined the required approximation precision empirically.\\
In this paper, we initiate the investigation of PANN as a standalone object.  Specifically, our contribution is two-fold. Firstly, we provide an explanation on the effect of approximate error in PANN. In particular, we discovered that (1) PANN is susceptible to some type of perturbations; and (2) weight regularisation significantly reduces PANN's accuracy. We support our explanation with experiments. Secondly, based on the insights from our investigations, we propose solutions to increase inference accuracy for PANN. Experiments showed that combination of our solutions is very effective: at the same precision, our PANN is 10\% to 50\% more accurate than state-of-the-arts; and at the same accuracy, our PANN only requires  
a precision of $2^{-9}$ while state-of-the-art solution requires a precision of $2^{-12}$ using the ResNet-20 model on CIFAR-10 dataset. 
\end{abstract}


%
\IEEEpeerreviewmaketitle
\newtheorem{theorem}{Theorem}[section]
\newtheorem{corollary}{Corollary}[theorem]
\newtheorem{lemma}[theorem]{Lemma}
\section{Introduction}
\label{sec:intro}

\par 
Machine learning (ML) is revolutionising different industries in recent years. Despite its widespread adoption, a key challenge in deploying ML is ensuring data privacy. One promising direction is to employ advance cryptosystems such as multiparty computation (MPC)~\cite{244032, peng2023autorep,baruch2023training} and homomorphic encryption (HE)~\cite{gilad2016cryptonets,lee2022low} during model training and model inference, resulting in privacy-preserving training and privacy-preserving inference. In this paper, we use the term Privacy-preserving machine learning (PPML) to cover both cases. 

\par While in principle these advance cryptosystems can be used to protect sensitive data used in the computation of any functionality, in practice they are designed to support basic arithmetics operations (i.e., addition and multiplication) or basic Boolean operations (i.e., AND and OR) only. This limitation makes it expensive to compute non-polynomial functions such as ReLU, sigmoid, and maxpool, commonly found in ML, and in particular, neural networks. Existing PPML address this limitation by replacing these functions with their polynomial approximations \cite{peng2023autorep, lee2022low, lee2021precise,baruch2023training}. In this paper, we refer to this type of neural network in which non-polynomial functions have been replaced with polynomials as polynomial approximation of neural network (PANN). Similar to existing MPC and HE based PPML, this paper focuses on privacy-preserving inference since privacy-preserving training using HE and MPC are still considered too inefficient in practice.

\par Naturally, the (implicit) use of PANN in privacy-preserving inference involves two design considerations, namely, the precision of the polynomial approximation and the accuracy of the PANN (compared with the underlying backbone model). Intuitively, the accuracy of inference cannot be guaranteed if the approximation introduces a large error. For example, early schemes with simple approximation and large errors can handle only simple model and tasks, as non-trivial errors generated by approximation result render the inference results useless in more complex neural networks \cite{gilad2016cryptonets, al2020towards, boemer2019ngraph,244032}. 

On the other hand, the use of high precision approximation guarantees inference accuracy at the cost of higher computational overhead \cite{lee2022optimization, lee2021high}. 
Although precise approximation can attain the same accuracy as the backbone model, it leads to much higher costs. Table \ref{tab:TimeCost} gives the time cost of PANN and their backbone models, and the corresponding accuracy using state-of-the-art techniques (without fine-tuning on PANN). It can be seen that the PANN on ResNet-20 with error bound $2^{-12}$ takes 71.2s to perform inference on the entire CIFAR-10 test set, whereas that with error bounds $2^{-8}$ only needs 28.2s. However, the inference accuracy is also significantly lower.

\par We make two additional remarks here. Firstly, while the overhead of using PANN is high, representing the non-polynomial functions using basic arithmetic or boolean operations is even more expensive. Secondly, we consider PANN as a standalone object, and compare their performance directly. When they are used in combination with advance cryptosystems, the difference will be further enlarged\footnote{For instance, in MPC, the function to be computed is first converted into a boolean or arithmetic circuit. If the circuit consists of $n$ gates, the resulting cost of the MPC is $O(n)$ but not necessarily linear in $n$.}. 

\par Some efforts have been made to increase PANNs' accuracy while reducing overheads by customizing the model structure for PANN or fine-tuning (training) the model on PANN \cite{gilad2016cryptonets, peng2023autorep, cryptoeprint:2023/162, baruch2023training} and achieve satisfactory performance.

However, there are several subtleties involved. Firstly, the polynomials' derivative is unbounded, making it hard to run gradients descent. Secondly, these models are not generic. Different approximation methods or even different precision for the same method can use different polynomials. This implies that these models must be trained separately for each method and precision. Thirdly, in many cases, popular model structures have been widely studied and provide pre-trained models, and thus training (or fine-tuning) over PANN or non-generic model structures can result in much higher validation cost and may not be desirable for real-world deployment. 

\begin{table}
    \caption{Time cost of backbone models (bb) and PANN with different error bound. Numbers in brackets indicate inference accuracy.}
    \label{tab:TimeCost} 
    \vskip 0.15in
\resizebox{\columnwidth}{!}{
  \begin{tabular}{c|c|c|c|c|c|c}
    \toprule
     & $2^{-8}$& $2^{-9}$& $2^{-10}$& $2^{-11}$& $2^{-12}$& bb\\
    \hline
    \multirow{2}{*}{ResNet-20} & 28.2s & 33.6s & 42.3s &   62.5s & 71.2s & 2.6s\\
    & (65.70) & (89.69) & (91.24) & (91.42) & (91.52) & (91.56) \\
    \hline
    \multirow{2}{*}{Shufflenetv2} & 38.1s & 44.8s & 50.5s & 79.8s & 89.7s & 3.3s\\
    & (11.94) & (32.08) & (84.96) &  (87.24)  & (88.13) & (88.60)\\
    \hline
    \multirow{2}{*}{DLA-34} & 128.8s &156.1s &163.5s &  251.8s  & 387.1s & 5.3s\\
    & (12.78) & (45.67) & (88.53) & (92.73) & (94.45) & (95.10)\\
    \hline
    \multirow{2}{*}{Mobilenetv2} & 221.6s &282.7s &295.5s &   479.1s  & 542.9s & 3.9s\\
    & (10.78) & (15.41) & (82.85) & (89.54) & (90.90) & (91.45)\\
    \bottomrule
  \end{tabular}
} 
\end{table}

\par To sum up, past researches appear to indicate that a trade-off seems inevitable: Low precision yields high efficiency but low accuracy. High precision yields high accuracy but unaffordable computation costs. The effect of approximation is relatively less understood. This serves as the motivation of this paper. Looking ahead, to better address this problem, as a first step, we separate PANN from the underlying cryptosystem, and investigate how approximation error affects inference accuracy in PANN. 

\subsection{Overview of Our Results} \label{sec:introoverview}
We initiate the investigation on how to reduce the impact on inference accuracy due to approximation errors in PANN. For the ease of writing, we say a neural network $\mathbb{F}$ is \emph{sturdy} if the PANN of $\mathbb{F}$ has good resistance against approximation errors. 
In other words, we can use less precise approximation for a sturdy neural network to maintain the inference accuracy. 

\par Specifically, we bound the lower bound of increased loss resulting from approximation errors and obtain the following two interesting results: 
\begin{itemize}
    \item Approximation errors on negative inputs of ReLU lead to larger loss than those on positive inputs of ReLU;
    \item The increased loss resulting from approximation errors will increase with larger weight decay and more training epochs.
\end{itemize}
To substantiate our analysis, we conduct tests and have observed the results across various model structures, datasets, and approximation methods.

\par Based on our findings, we present two solutions to improve ``sturdiness" of neural networks, namely, the negative inputs leakage method and reducing the use of weight regularization. Our first solution is based on the insight that approximation errors on ReLU's negative inputs lead to more loss. In more detail, 
our first solution introduces perturbations on the negative inputs of ReLU during baseline training. Our second solution is based on the findings that weight regularization is harmful to ``sturdiness”. Intuitively, we can improve “sturdiness” by removing weight regularization during training. However, weight regularization is crucial in providing generalization. Our second solution is considered a trade-off: we use minimal weight regularization and combine it with Mixup to maintain an acceptable accuracy in the backbone model yet greatly increase its ``sturdiness”.
We conduct extensive experiments to illustrate the effectiveness of our solutions. For instance, we achieve the state-of-the-art non-fine-tuning accuracy on CIFAR-10 with only 40\% to 60\% time cost (as shown in Figure 1). We improve the accuracy of the state-of-the-art ReLU replacement scheme~\cite{peng2023autorep} by 3\%. Compared to previous schemes, our methods have the following advantages:
\begin{itemize}
    \item Our obtained models are based on generic model structures rather than being specifically fine-tuned for particular PANNs. That means the same model can perform as the backbone (or pre-trained) model for various approximation precision and methods to enhance their performance without adjustment. This lowers the training cost and ensures better compatibility with other approaches.
    \item Our training methods are also general without restricting to particular model structures or approximation methods. This allows for a variety of fine-tuning-based schemes to be further optimized with our approaches.
\end{itemize}

\subsection{Outline}
\par The rest of the paper is organized as follows.
Section 2 gives the theorem and explanation of how approximation errors affect PANN.
Section 3 introduces the related works. Section 4 describes the preliminary of our study. 
Section 5 provides solutions to improve low-precision PANN's accuracy. 
Section 6 gives the experiment results.
Finally, in Section 7, we conclude this paper and discuss potential avenues for future research.

\begin{figure*}[!t]
\centering
\subfloat[ResNet-20]{\includegraphics[width=1.7in]{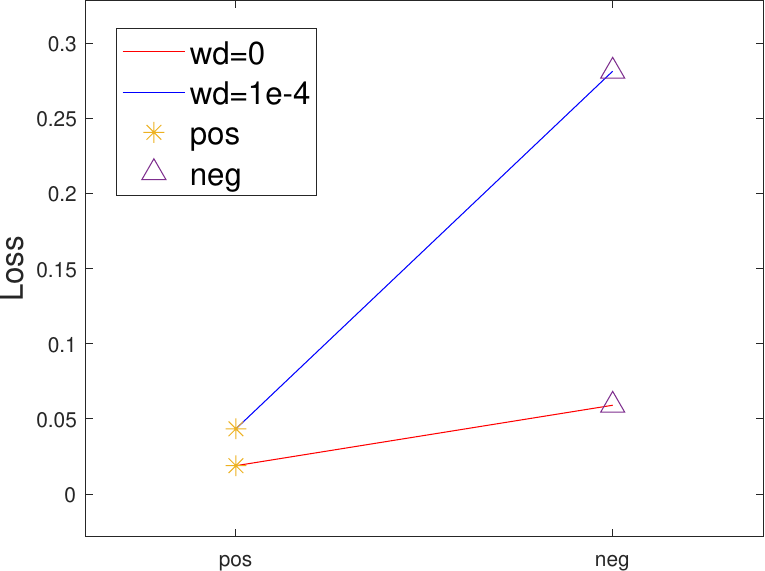}%
\label{fig_first_case}}
\hspace{0mm}
\subfloat[Shufflenetv2]{\includegraphics[width=1.7in]{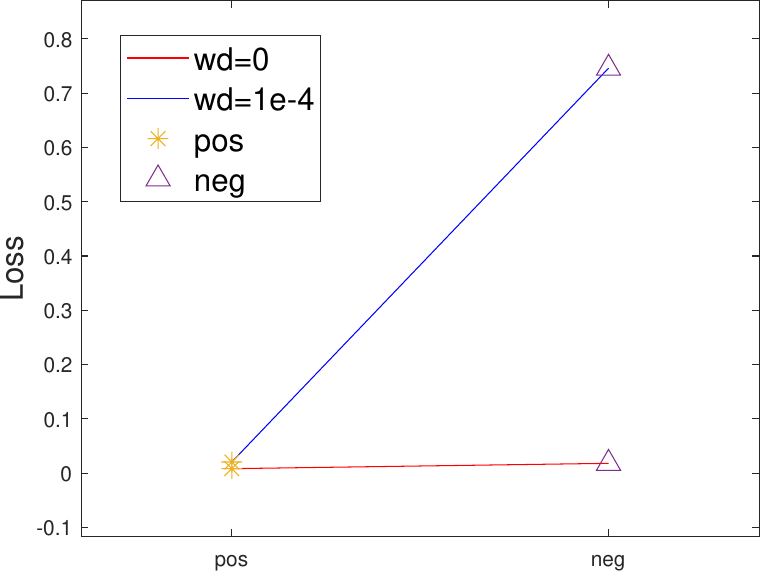}%
\label{fig_first_case}}
\hspace{0mm}
\subfloat[DLA-34]{\includegraphics[width=1.7in]{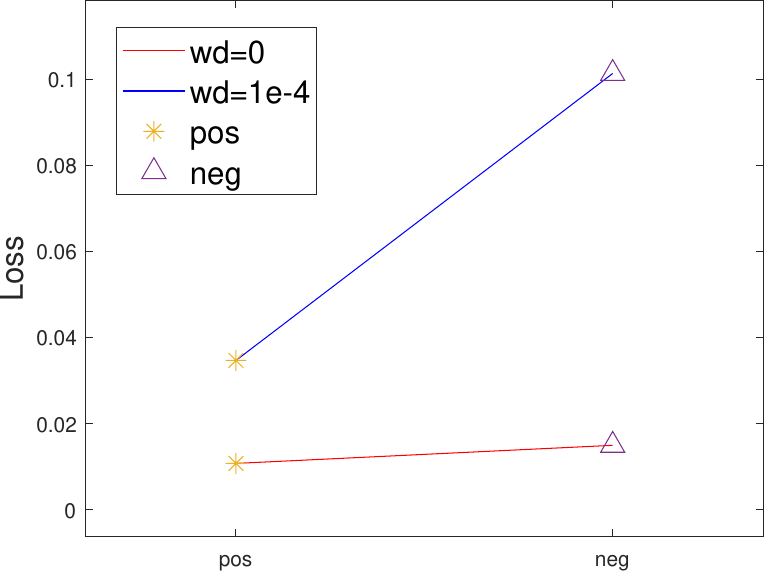}%
\label{fig_first_case}}
\hspace{0mm}
\subfloat[Mobilenetv2]{\includegraphics[width=1.7in]{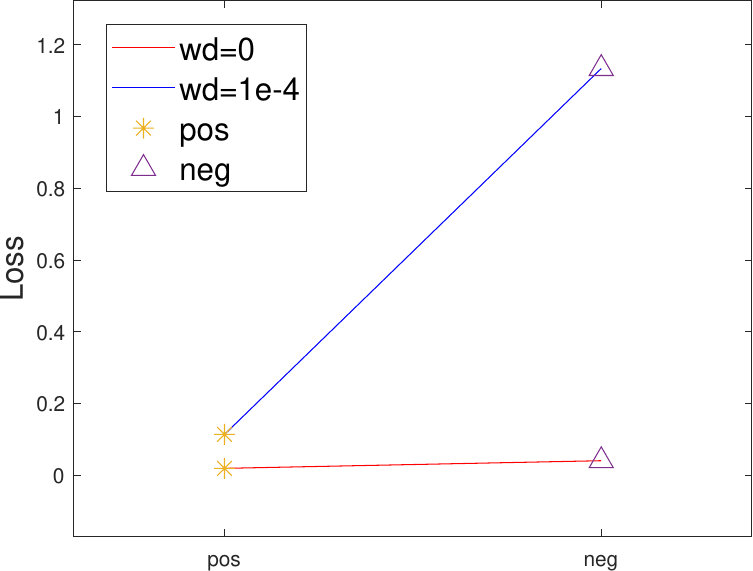}%
\label{fig_second_case}}
\caption{Testing loss increment (compared to baseline) of PANN~\cite{lee2022low} approximating ReLU with positive inputs (pos) and negative inputs (neg) on models trained with weight decay (wd) 0 and 1e-4. It shows that weight decay can amplify the loss increment caused by approximation. Additionally, approximating ReLU with negative inputs results in more loss compared to approximating ReLU with positive inputs.}
\label{fig:loss_wd_pn}
\end{figure*}

\section{Effects of Approximation Errors on PANN}\label{sec:Stationaryshift}

\par In this section, we present our main theorem and discuss the factors that affect the loss increment caused by approximation errors. Since for most PANN schemes, activation functions, especially the ReLU function, are the main components requiring approximation, we will focus specifically on the approximation of ReLU functions.

\par Considering a two-layer neural network with ReLU activation functions. The input point is a $d$-dimentional vector $\bfit{x} \in \mathbb{R}^{d}$. The weights is $\bfit{W} \in \mathbb{R}^{n*d}$, let $\bfit{z}=\bfit{W}\bfit{x}$ and $\bfit{y} = \operatorname{ReLU}(\bfit{z})$. For a loss function $L$, weight decay parameter $\lambda$, and learning rate $\eta$, we define the loss function for SGD with with L2 regularization at epoch $t$ as:
\begin{equation} \label{eq:loss_nowd}
    f(\bfit{W}_{t})= L(\bfit{W}_{t}) + \frac{\lambda}{2\eta} \|W_t\|^{2} 
\end{equation}
For the ease of writing, we use $h(\operatorname{ReLU}(z)) = f(\bfit{W},\bfit{X})$ to represent the same loss but with inputs $\operatorname{ReLU}(z)$ where $z$ is a value in $\bfit{z}$, such that the loss with error $\epsilon$ is $h(\operatorname{ReLU}(z)+\epsilon)$. We use $\Delta h$ to represent the increased loss caused by error $\epsilon$ (Note that $\Delta h = \Delta f$). Now we give two theories and the corresponding assumptions.

\subsubsection*{Approximation errors on the negative inputs
of ReLU increase more loss than those on positive inputs} A question that remains unexplored is which part of approximation errors lead to more accuracy reduction (or a higher loss). Here we give Theorem \ref{th:pos_neg_loss} which indicates that errors on negative inputs of ReLU result in larger loss and therefore poorer accuracy.
\begin{theorem} \label{th:pos_neg_loss}
    Let two convex function $h_1: \mathbb{R} \to \mathbb{R}$ and $h_2: \mathbb{R} \to \mathbb{R}$ differentiable in $(-\infty,0)\cup(0,\infty)$. Let $\overline{z}_1 \in \mathbb{R}$ and $\overline{z}_2 \in \mathbb{R}$ and $\overline{z}_1<0, \overline{z}_2>0$, such that $h_1'(\overline{z}_1))=h_2'(\overline{z}_2))=0$. Let $\epsilon \in \mathbb{R}$ be a small value, we have:
    \begin{equation}
        \lim_{\epsilon \to 0} (\Delta h_1 - \Delta h_2) = \epsilon \cdot \mathop{\operatorname{inf}}_{z>0}\phi_{0,h_1}(z) 
    \end{equation}
\end{theorem}
The theorem suggests that when approximating ReLUs with negative inputs, the loss is at least as large as approximating those with positive inputs. This implies that it's beneficial to reduce the approximation of  ReLU with negative inputs or to develop specific training strategies for negative inputs. 
\par In Figure \ref{fig:loss_wd_pn}, we present the increase in loss when only approximating ReLU with negative (neg) or positive (pos) inputs.
It illustrates the phenomenon that approximating ReLUs with negative inputs results in more loss than approximating ReLUs with positive inputs (when errors are small). This can be observed in various model structures and provides evidence for our theorem.

\subsubsection*{Weight decay enlarges the loss increasing resulting from approximation error as training progresses} Another important question is which factors in backbone model training affect the ``sturdiness”. Now we provide Theorem \ref{th:loss_bound_wd_t}, which proves that weight decay is responsible for weak sturdiness.
\begin{theorem} \label{th:loss_bound_wd_t}

    Let function $L: \mathbb{R} \to \mathbb{R}$ be a $\mathcal{L}$-smooth function and $h$ be a convex function with $h'(\operatorname{ReLU}(\overline{z}))=0$.
    Assume $h$ is differentiable in $(-\infty,0)\cup(0,\infty)$. Let $\epsilon \in \mathbb{R}$ be a small value. $\mathbb{E}[ \nabla L(\bfit{W}, X) - \nabla L(\bfit{W})  ] = 0$, $\mathbb{E}[\| \nabla L(\bfit{W}, X) - \nabla L(\bfit{W}) \|^{2} ] \leq \sigma^{2}$, $\| \nabla l(\bfit{W}) \| \leq G$ for any $\bfit{W}$, and $\eta \leq \frac{C}{\sqrt{t+1}} $.
    If the model is trained $t'+1$ epochs after it has reached the stationary point $\bfit{W}^{*}$ where $L(\bfit{W}^{*}, X)=L^{\star}$, the increased loss $\Delta L=h(\operatorname{ReLU}(\overline{z}+\epsilon)) - h(\operatorname{ReLU}(\overline{z}))$ caused by error $\epsilon$ is bounded by:
        \begin{equation}
        ||\epsilon||^2 \cdot\mathbb{E}[|| \mathop{\operatorname{inf}}_{z>\overline{z}}\phi_{\overline{z}}(z)||^2] \leq \mathbb{E}[||\Delta h||^2]
        \end{equation}
        where 
        \begin{equation}\notag
            \operatorname{sup} \mathbb{E}[||\operatorname{inf}_{z>\overline{z}}\phi_{\overline{z}}(z)||^2] \leq ||h'_{+,t}(\overline{z})||^2 + \frac{t'(C_1+C_2)}{d\sqrt{t+t'+1}}
        \end{equation}
        and:
        \begin{align}\notag
            C_1 &=\frac{L(\bfit{W}_{0}) - L^{\star} }{C} \\
            C_2 &= C (\mathcal{L} + \lambda) ((G+  \sup(\lambda\|W\|) )^{2} + \sigma^{2}),
        \end{align}
\end{theorem}
\par The theorem demonstrates that the loss will increase with larger weight decay $\lambda$ and epoch $t'$. We have observed and confirmed the effect of these two factors on PANN. 
\par An example of PANN~\cite{lee2022low} on models trained with different weight decay is shown in Figure~\ref{fig:wd_PANNandori_training}. It illustrates that though the backbone (bb) models have similar accuracy, the accuracy of PANN reduces when weight regularization ($\lambda$) is applied as the epochs ($t'$) increase. Larger weight decay can make this reduction quicker ($\lambda$ increases). Furthermore, Figure \ref{fig:loss_wd_pn} illustrates the increase of testing loss of PANN~\cite{lee2022low} with different backbone and different weight decay, indicating that weight decay can elevate the testing loss in all cases. (Still, note that the error is small.)

\begin{figure}[h]
\centering
        {\includegraphics[width=\columnwidth]{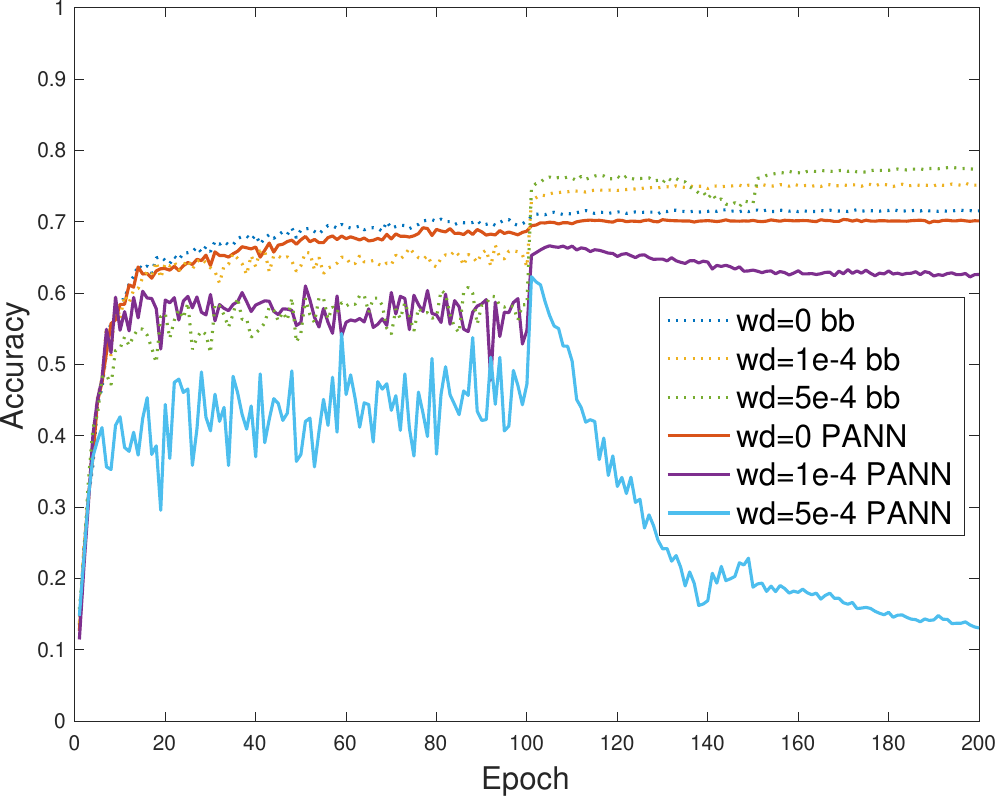}}
    \caption{Backbone models (bb) trained with different weight decay (wd) have similar accuracy (dot line). However, large weight decay can reduce their accuracy on PANN as the epoch increases. (ResNet-18, CIFAR-100, precision $2^{-9}$)}
    \label{fig:wd_PANNandori_training}
\end{figure}

\par This phenomenon can be observed in different model structures, datasets, approximation error bounds (Table \ref{tb:results}, Table \ref{tb:results_2}), and approximation methods. We believe the above theorems and analysis provide a plausible explanation on how approximation errors reduce the accuracy of PANN. (Detailed proof of the two theorems is put in Appendix \ref{app:proof}.) 
\subsubsection*{An intuitive explanation} 
Here we give an intuitive explanation based on the definition of stationary point to help understand the two theorems. A stationary point is where the function's derivative is zero, and is a necessary condition for minimizing training losses. For stochastic gradient descent (SGD) method without weight regularization, we have $f=L$. The model weights for epoch $t+1$ are updated by:
\begin{equation}\label{eq:L2update}
    \bfit{W}_{t+1}  = \bfit{W}_{t} -  \eta \frac{\partial L(\bfit{W}_{t})}{\partial \bfit{W}_{t}}, 
\end{equation}
then model training is searching for the smallest loss which have $\frac{\partial L(\bfit{W}_{t})}{\partial \bfit{W}_{t}}=\mathbf{0}$.
\par Let's focus on a single input sample $x_{1}$ and the $j$-th column of $\bfit{W}$ (represented as $\bfit{W}_{j}$). Let $z$ be a single value in vector $\bfit{z}$ and $y$ be the corresponding value in $\bfit{y}$ that satisfies $y=\operatorname{ReLU}(z)=\operatorname{ReLU}(\bfit{W}_j x_{1})$. Assume a non-zero stationary point $\bfit{W}^*$ is found at epoch $t$. For $\lambda =0$, by combining Equation \ref{eq:loss_nowd}, we can get:
\begin{equation}\label{eq:loss_backbone}
     \frac{\partial L}{\partial \bfit{W}^*_{j}} = \frac{\partial L}{\partial y}\frac{\partial y}{\partial z}\frac{\partial z}{\partial \bfit{W}^*_{j}} = 0
\end{equation}
In this condition, the neural network has minimal loss and is the most stable because small input changes will have little effect on the outputs. However, for PANN with approximation errors, the value $y$ is changed to:
\begin{equation} \label{eq:y_epsilonz}
    \widetilde{y} = y + \epsilon_{relu}(z)
\end{equation}
This leads to a difference between PANN and the backbone (or pre-trained model) that a stationary point in the original model may not be the stationary point in PANN. We discuss it in two cases,  ReLU with negative or zero inputs ($z<0$) and ReLU with positive inputs ($z>0$).
\par For $z <0$, we have $\frac{\partial y_{-}}{\partial z}=\frac{\operatorname{ReLU}(z)}{\partial z}=0$. Due to the property of ReLU, $L$ is non-differential at the point $z=0$. So we have $\frac{\partial y_{-}}{\partial z}\neq \frac{\partial y_{+}}{\partial z}$.

From Equation \ref{eq:loss_backbone} we can know $\frac{\partial L}{\partial \bfit{W}^*_{j}}=0$ and the stationary point holds for all $z \leq0$ in the backbone model. However, in PANN $\operatorname{ReLU}(z)$ is replaced with $\operatorname{ReLU}(z)+\epsilon$. Assume $\frac{\partial L}{\partial y}\frac{\partial z}{\bfit{W}^*_{j}} \neq 0$ (which is the most cases), then $\bfit{W}^*$ may not be a stationary point of PANN, because from \ref{eq:y_epsilonz} we can get:
\begin{equation} \label{eq:partial_y_z_lt0}
    \frac{\partial \widetilde{y}}{\partial z} = \frac{\partial y_{+}}{\partial z} + \frac{\partial \epsilon}{\partial z}
\end{equation}
By combining Equation \ref{eq:loss_backbone} and \ref{eq:partial_y_z_lt0}, we get:
\begin{equation}\label{eq:loss_PANN_with_error_zlt0}
     \frac{\partial \widetilde{L}}{\bfit{W}^*_{j}} = \frac{\partial L}{\partial y}\frac{\partial \epsilon}{\partial z}\frac{\partial z}{\bfit{W}^*_{j}}
\end{equation}

As long as $\frac{\partial y_{+}}{\partial z} + \frac{\partial \epsilon}{\partial z} \neq 0$, we have $\frac{\partial \widetilde{L}}{\bfit{W}^*_{j}} \neq 0$, which means $W^*$ is not the stationary point of PANN. The loss is not a minimum and is not sufficiently stable.

\par Differently, for $z > 0$, we have $\frac{\partial y}{\partial z} =1$ in backbone model. By combining Equation \ref{eq:loss_backbone} we have:
\begin{align}\label{eq:devi_LtoW2}
    \frac{\partial L}{\partial y}\frac{\partial z}{\bfit{W}^*_{j}} =0
\end{align}
For PANN, from Equation \ref{eq:y_epsilonz} we can get:
\begin{equation}\label{eq:partial_yz_zge0}
    \frac{\partial \widetilde{y}}{\partial z} = \frac{\partial y}{\partial z} + \frac{\partial \epsilon}{\partial z} = 1+\frac{\partial \epsilon}{\partial z}
\end{equation}
Because ReLU is differentiable for positive inputs, when the approximation error is small enough, the neural network can be seen as linear and $\frac{\partial L}{\partial y}\frac{\partial z}{\bfit{W}^*_{j}}$ remains unchanged. By combining Equation \ref{eq:loss_backbone} and \ref{eq:partial_yz_zge0}, we have:
\begin{align}\label{eq:devi_LtoW_PANN2}
    \frac{\partial \widetilde{L}}{\partial W_{j}} = (1+\frac{\partial \epsilon}{\partial z})\frac{\partial L}{\partial y}\frac{\partial y}{\partial z}\frac{\partial z}{\bfit{W}^*_{j}} = 0
\end{align}
That means the backbone model and PANN share the same stationary point when errors are added to positive inputs ($z_{j}> 0$) of ReLU, which reduces the effect of approximation errors. Besides, even if the weight was driven away from $W^*$ during the training, an item $-\frac{\partial L(W,x_1)}{\partial W}$ in weight updating will fix the stationary point towards $W^*$. 

When weight decay is involved, it has been proved that weight decay can lead to large gradient norms~\cite{xie2024overlooked}. It may lead to the following problems: (1) Increase the norm of $z$, thus increasing the amplitude of some noise proportional to the $||z||$; (2) Increase $\frac{\partial L}{\partial z_{+}}$ when $z<0$. Because the stationary points $\frac{\partial L}{\partial z_{-}} =0$ always hold, $\frac{\partial L}{\partial z_{+}}$ may be trained towards away from the stationary point. Weight decay will accelerate this process (by destroying patterns established before $z$ is negative).

\section{Related Works and Background}
\subsection{Approximation Methods for Neural Networks} \label{sec:related_app}

Replacing non-polynomial functions with polynomials has become a desirable tool for applying cryptography to neural networks. Early studies often adopt simple (low-degree) polynomials in approximation. For example, CryptoNets and HCNN utilized the square function to replace activation functions \cite{gilad2016cryptonets, al2020towards}. CryptoDL used degree 2 and degree 3 polynomials to approximate the ReLU function \cite{hesamifard2019deep}, while Faster Cryptonets leveraged minimax approximation with degree two \cite{chou2018faster}. nGraph-HE and Delphi adopt quadratic approximations to approximate activation functions \cite{boemer2019ngraph,244032}. These methods can handle only simple model (i.e. less than 10 layers) and tasks (i.e. classification on MINIST) because the non-negligible errors accumulated throughout layers can seriously affect the model outputs.

\par To address this, further studies proposed highly precise approximation to reduce the approximation error and can achieve the same accuracy as the backbone model~\cite{lee2022low, lee2021privacypreserving}. A common method for these precise approximations is Minimax approximation~\cite{lee2022optimization, lee2021high,lee2021minimax}, which combines many small-degree polynomials for approximation to reduce the computation cost. However, the reduced computational overhead of precise approximation is still unbearable. Increasing PANNs' accuracy while reducing overheads remains an important issue.

\subsection{Fine-tuning-based Solutions}
When seeking to improve the performance of PANN, many studies fine-tune (or train) models directly on PANN. Earlier studies such as CryptoNets and Faster Cryptonets directly train models on neural networks with low-degree polynomials activation functions~\cite{gilad2016cryptonets, chou2018faster}. The models obtained are essentially new models with a customized design for PANN (both in terms of structure and parameters) rather than being approximations of commonly used models. The limitations of simple polynomials activation and the customized structure make these models hard to tackle complex tasks. To better suit practical applications, recent studies tend to input pre-trained models as a starting point and then fine-tune them on PANN. For example, SNL and AutoReP fine-tune models after replacing a partial of ReLU functions with polynomials with one or two degree~\cite{cho2022selective, peng2023autorep}. AutoFHE fine-tunes models on precise PANN to achieve better performance and lighten the requirements for precision~\cite{cryptoeprint:2023/162}. Another study fine-tuning models after carefully adjusting the model structure for low-degree polynomial approximations has also obtained a high accuracy~\cite{baruch2023training}.
\par Despite its efficiency in improving PANN accuracy, several issues are involved. Firstly, polynomials are not appropriate for training when used as activation functions. The derivative of these polynomials is unbounded, which can lead to unstable values during gradient descent and get unexpected results (i.e. the parameter instability issues in AutoReP~\cite{peng2023autorep}). Secondly, a fine-tuned model is customized for the weights and polynomials it's trained on. The model must be retrained whenever you want to adjust the PANN's approximation precision and overhead (discussed in Section \ref{sec:PANNs}), which will result in additional training costs. Thirdly, generic model structures and activation functions are usually well-validated. Designing extra network structures for PANN results in higher design and test costs and is undesirable for other applications. Therefore, our study prefers to solve PANN's accuracy reduction problems on generic model structures crossing different approximation methods and precision requirements without fine-tuning.

\subsection{ReLU Replacement Schemes}
Reducing the number of activation functions in PANN can significantly reduce computation costs since they are the main part requiring approximation and cryptographic operations. SNL first achieves this by designing a parametric ReLU (PReLU) and applying a loss function to gradually adjust ReLU functions to linear functions~\cite{cho2022selective}. This method can reduce the ReLU number in a ResNet-18 model from 491.52K to 49.9K while maintaining 73.75\% accuracy (baseline 77.8\% on CIFAR-100). AutoRep improves the method by introducing parameterized discrete indicator function, hysteresis loop update function, and distribution-aware polynomial approximation~\cite{peng2023autorep}, which improves the accuracy to 75.48\%. 

\par Though ReLU replacement can significantly reduce the computational overhead of PANN, two factors reduce the accuracy of obtained PANNs in this process. First, the structure of the pre-trained model is different from PANN. However, the pre-training parameters are not resistant to perturbations caused by structural changes. Although this effect can by reduced by fine-tuning, there is still much room for improvement. Second, the activation functions during backpropagation and the activation functions in the final obtained models are also different. This further reduces PANN's accuracy. 
\par Actually, the two factors point to the same problem that the model is not resistant to approximation errors (we say ``poor sturdiness”). While our solutions, which increase model sturdiness based on general structures, can serve as a supplementary to ReLU replacement schemes. By taking our models as pre-trained and teacher models, the performance of SNL and AutoRep with different ReLU counts can all be further improved. 

\subsection{ReLU Protocol based on Truncation}

Bicoptor~\cite{zhou2023bicoptor} first proposed an efficient three-party computation (3PC) protocol for ReLU based on truncation and Bicopter2.0~\cite{zhou2023bicoptorv2} further improved it. For a fix-point input $x$ with $l_x$ bit length, the protocol will right-shift the complement code of $x$ for $l_x$ times with truncation and get $l_x$ shares. By securely checking if zero appears in these shares, the protocol will determine the sign of $x$ (i.e., $x = 0b00001010$ right shift 4 bit will get $0$) and enable the computing of $\operatorname{ReLU}(x)$. Since this process involves transferring $lx$ truncation results, the communication cost will increase with $l_x$ ($(+2)l$ bit for Bicoptor, $(l_x+1)(l_x+1)$ for Bicoptor2.0). 

\par This means choosing a suitable (small) $l_x$ in ReLu protocol can effectively reduce the communication overhead at the cost of computation precision. We considered this error also a kind of approximation error. Though it's different from polynomial approximation, the effects is still in line with the results of our analyses and can apply our solutions.

\subsection{Deep Learning with Differential Privacy}
Differential privacy (DP) protects the privacy of individual tuples in a database~\cite{dwork2006calibrating,duchi2013local,dwork2014algorithmic,mironov2017renyi}. It ensures that attackers cannot infer the membership of any tuple from the released data while the statistical properties are preserved for usage. By applying DP in machine learning, models can be trained without exposing sensitive information~\cite{abadi2016deep,wei2020federated,du2023dp,yang2023privatefl}. DP-based machine learning does not require time-consuming cryptographic computation and thus is highly practical. However, several serious issues are involved. Firstly, the statistics information of the dataset is also valuable. Especially when the data collection is costly, data owners may be reluctant to disclose any information. Secondly, DP cannot absolutely protect user privacy. Attackers can still get information about some users, though with a low probability. Thirdly, DP aims to hide data in a database, which should contain a large amount of data. This requirement cannot be satisfied in scenarios like privacy-preserving inference. 

\par While cryptographic methods can protect the information of all inputs (both statistical information and individual tuples) and have no requirements for the dataset. In more detail, cryptographic schemes can ensure that every ciphertext in the computation process does not reveal the plaintext information if the attacker doesn't have the key. Cryptographic techniques such as FHE and MPC can compute the data in the ciphertext format, and the result is generated and stored in ciphertext, which can only be revealed after decryption.

\subsection{Side-effects of Weight Decay}

Weight decay is a powerful regularization technique that has been very widely used in training deep neural networks. It can improve model generalization and accuracy~\cite{zhang2018three}. However, recent studies have shown that this benefit comes with side effects~\cite{ziyin2022exact, xie2024overlooked}. Weight decay has been proved to increase bad minima when applied in training neural networks~\cite{ziyin2022exact}. Further study has proved that weight decay can lead to large gradient norms at the final phase of training, which often indicates bad convergence and poor generalization~\cite{xie2024overlooked}. This large gradient norm problem is significant in the presence of scheduled or adaptive learning rates and has been recognized as harmful to neural networks. In our study, we have also confirmed that weight decay reduces the sturdiness of the model, which to some extent, supports these arguments.

\section{Preliminary}
\subsection{Notation}
\par We denote a neural network as $\mathbb{F}$, benign inputs and label as ($x$, $y$). we use $\widetilde{\cdot}$ (i.e., $\widetilde{y}$) to represent the values in PANN (with approximation errors). We use $W_i$ to denote the $i$-th row vector of a matrix $W$. A small value is represented by $\epsilon$. 
The error bound for approximation is $2^{-\beta}$ determined by $\beta$. The approximation of $\operatorname{sgn}$ and $\operatorname{ReLU}$ are denoted by $\operatorname{Appsgn}$ and $\operatorname{AppReLU}$ respectively.

\subsection{Polynomial Approximation of Neural Networks}\label{sec:PANNs}
\par Polynomial Approximation of Neural Networks (PANN) is the neural network replacing all (or part of) non-polynomial functions with polynomial approximations. 
Common approaches for approximation in PANNs include low-degree polynomials, Taylor polynomials, Minimax approximation, etc. Among these methods, low-degree polynomial is the fastest but yields low accuracy, thus generally requires customized model structures and fine-tuning models directly on PANN. In contrast, Minimax approximation with high-degree polynomials is slow but can yield negligible approximation errors, which can achieve PANN with the same accuracy as backbones. 

\par In this paper, we refers to $p(z)$ as a polynomial approximation of a function $g(z)$. The approximation error of this approximation is defined as:
\begin{equation}
    \epsilon_{g}(z) = p(z)-g(z)
\end{equation}

\subsubsection{Simple approximation with low-degree polynomials}

\par Since training directly on PANN with simple polynomials~\cite{gilad2016cryptonets, chou2018faster} is not an approximation of the existing model structure and parameters (discussed in Section \ref{sec:related_app}) and PANNs replacing all nonlinear functions are usually targeting on high-degree polynomials or precise approximation~\cite{cryptoeprint:2023/162,baruch2023training}. We only introduce ReLU replacement schemes which replace a partial of non-polynomial functions~\cite{cho2022selective, peng2023autorep}.

\par ReLU replacement aims at replacing a partial of ReLU activation functions with polynomial functions such as $p(z)=z$~\cite{cho2022selective} and $p(z)=0.14z^{2}+0.5z+0.28$~\cite{peng2023autorep}. Given a pre-trained model, this process can be achieved by the following three steps: (1) Replace all activation functions (ReLU) $g(z)$ in it with: 
\begin{equation}
    \sigma(z) =c\cdot g(z) + (1-c)\cdot p(z)
\end{equation}
(2) Train the model to reduce the value of $|c|$ by designing a proper loss function, while fine-tuning the pre-trained weights to suit the new activation functions; (3) Binarize $c$ depending on if the value $c$ satisfies some conditions (e.g.: if $c$ is smaller than a threshold value), and get the final activation function equal to $g(z)$ or $p(z)$; (4) Freeze $c$ and fine-tune the model based on the obtained activation functions.

\par ReLU replacement schemes will repeat the above steps until the ReLU counts are reduced to a certain number. This process results in two problems. First, the pre-trained model parameters were designed for the original structure using ReLU functions, and it was never intended to use polynomials from the beginning. Even with fine-tuning, it is difficult to achieve the same accuracy as the baseline (especially when the number of ReLUs is very small).

Second, the model weights are fine-tuned for the function $g'(z) =c\cdot g(z) + (1-c)\cdot p(z)$, while the activation function in the final result is $g(z)$ or $p(z)$. This introduces an approximation error:
    \begin{equation}
        \epsilon_g(z) = \left\{
    \begin{aligned}
    &(1-c)\cdot g(z) - (1-c)\cdot p(z), &\sigma(z)=g(z)\\
    &-c\cdot g(z) + c\cdot p(z), &\sigma(z)=p(z)\\
    \end{aligned}
    \right.
    \end{equation}
These two factors together result in the accuracy reduction in the obtained PANN, and they both indicate the same issue: the model is not sufficiently resistant to approximation errors, which we call poor “sturdiness”.

\subsubsection{Precise approximation with Minimax approximation}
\par In Minimax approximation, $p(z)$ is composed by small degree polynomials: $p = p_{k}\circ p_{k-1}\circ \dots \circ p_{1}$. The degree and parameters of small-degree polynomials are chosen according to approximation precision. The approximation precision is usually determined by a parameter $\beta$, which means that the absolute value of approximation error $\epsilon_{g}(z)$ is bounded by $2^{-\beta}$: $|\epsilon_{g}(z)|<2^{-\beta}$. An $\beta$-close polynomial approximation of a function $g(z)$ satisfies:
\begin{equation} \label{eq:app_function}
    |p(z)-g(z)|\leq 2^{-\beta}
\end{equation}

\begin{figure}[h]
\subfloat{\includegraphics[width=0.5\columnwidth]{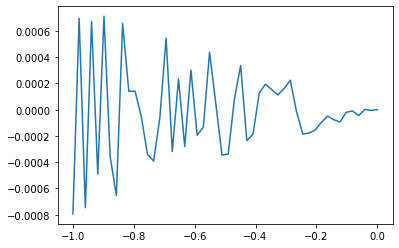}}
      \hfill
\subfloat{\includegraphics[width=0.465\columnwidth]{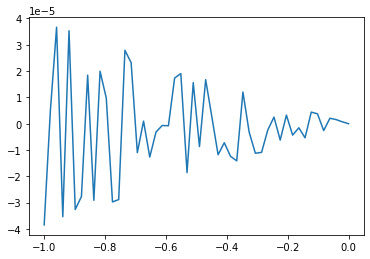}}
    \caption{Approximation errors for the Minimax approximation of $\operatorname{sgn}(z)$ on $[-1,-\epsilon]$ with precision $2^{-10}$ (left) and $2^{-14}$ (right)}
    \label{fig:approxsgn}
\end{figure}
When used in cryptographic schemes, parameter $\beta$ (or the corresponding ) is an important indicator of computation cost. Because a large $\beta$ (high precision) requires a $p(z)$ with a large degree, which involves a large amount of multiplication in cryptography computation. This phenomenon is extremely significant in HE-based schemes since successive multiplications require bootstrapping (decryption and re-encryption in ciphertext), which is the most expensive part in HE. Therefore, reducing the effects of approximation errors and thus allowing for low precision (or degree) is an important topic.

\par For most PANN schemes, activation functions, especially the ReLU function, are the main part requiring approximation and naturally are the main cause of approximation errors. Since the ReLU function itself is hard to approximate in Minimax directly, it's typically approximated by the sgn function using the formula:
\begin{equation} \label{eq:sng_to_relu}
    \operatorname{AppReLU}(z) = \frac{z+z\operatorname{Appsgn}(z)}{2}
\end{equation}

\par Define the error for approximating the function $\operatorname{sgn}(z)$ as:
\begin{equation}
    \epsilon_{sgn}(z) = \operatorname{Appsgn}(z) - \operatorname{sgn}(z)
\end{equation}
Then the approximation error for ReLU approximation is:
\begin{equation} \label{eq:error_sgntorelu}
        \epsilon_{relu}(z) = 
        \frac{\epsilon_{sgn}(z)  z}{2}
\end{equation}

If the $\operatorname{Appsgn}$ satisfies $\beta$-close that $|\epsilon_{sgn}| < 2^{-\beta}$, we can get:
\begin{equation} \label{eq:errorlimit_sgntorelu}
    |\epsilon_{relu}(z)| < 2^{-\beta}|z|
\end{equation}
Equation \ref{eq:errorlimit_sgntorelu} illustrates that the magnitude of approximation errors is proportional to the magnitude of $z$ (the input to ReLU).

\subsection{Mixup}
Mixup is the method that introduces interpolated samples to the training set \cite{zhang2018mixup}, which was proved to help model robustness and generalization \cite{zhang2021does}. For parameter $\lambda$ and samples $(x,y)$ and $(x', y')$ randomly selected from the training set, the interpolated sample is:
\begin{equation} \label{eq:mixup_input}
        \widetilde{x} = \lambda x + (1-\lambda) x',\quad
        \widetilde{y} = \lambda y + (1-\lambda) y'
\end{equation}

\begin{table*}[ht]
\caption{Accuracy of PANN on ResNet-20, Shufflenetv2, DLA-34, Mobilenetv2 trained with weight decay (wd) 0, 1e-4 and 5e-4. We give Top-1 accuracy of the backbone (baseline) and PANN with precision $2^{-\beta}$ ($2^{-8}$ and $2^{-9}$). Training methods include Vanilla, Mixup, and robustness-based solution NGNV. The results with the highest accuracy for each model and approximation precision are bolded.}
\label{tb:results}
\centering
\begin{tabular}{c|c|ccc|ccc|ccc|ccc}
\hline
 &  & \multicolumn{3}{c|}{Vanilla} & \multicolumn{3}{c|}{Mixup} & \multicolumn{3}{c|}{NGNV} & \multicolumn{3}{c}{Mixup+NGNV} \\ \hline
 & \diagbox[width=4em]{$\beta$}{wd} & 0 & 1e-4 & 5e-4 & 0 & 1e-4 & 5e-4 & 0 & 1e-4 & 5e-4 & 0 & 1e-4 & 5e-4 \\ \hline
\multirow{3}{*}{ResNet-20} & 8 & 86.84 & 65.70 & 12.07 & 87.91 & 76.95 & 11.77 & 87.77 & 75.33 & 11.41 & {\ul \textbf{88.79}} & 78.82 & 12.53 \\
 & 9 & 89.60 & 89.69 & 58.81 & 90.52 & 89.67 & 43.23 & 90.61 & 90.43 & 73.82 & {\ul \textbf{91.02}} & 90.48 & 47.67 \\ \cline{2-14} 
 & Baseline & 90.39 & 91.57 & 92.14 & 91.30 & 91.60 & 92.17 & 91.10 & 91.70 & 92.49 & 91.49 & 91.59 & 92.16 \\ \hline
\multirow{3}{*}{Shufflenetv2} & 8 & 16.11 & 12.25 & 11.94 & 11.26 & 12.35 & 13.23 & {\ul \textbf{84.44}} & 20.96 & 10.09 & 72.54 & 31.62 & 13.20 \\
 & 9 & 58.60 & 32.22 & 32.08 & 36.35 & 24.29 & 25.41 & {\ul \textbf{87.65}} & 85.36 & 23.15 & 84.10 & 81.57 & 28.02 \\ \cline{2-14} 
 & Baseline & 88.54 & 90.22 & 88.60 & 89.08 & 90.19 & 87.62 & 88.66 & 90.69 & 89.60 & 89.10 & 90.54 & 88.32 \\ \hline
\multirow{3}{*}{DLA-34} & 8 & 91.69 & 69.02 & 12.78 & 90.45 & 58.01 & 13.23 & 92.10 & 50.94 & 13.58 & {\ul \textbf{93.07}} & 52.80 & 13.45 \\
 & 9 & 92.28 & 88.79 & 45.67 & 93.52 & 87.00 & 36.80 & 93.64 & 91.94 & 38.73 & {\ul \textbf{94.82}} & 91.92 & 26.46 \\ \cline{2-14} 
 & Baseline & 93.43 & 94.38 & 95.10 & 94.96 & 95.12 & 95.41 & 93.85 & 94.92 & 95.21 & 95.09 & 95.62 & 95.58 \\ \hline
\multirow{3}{*}{Mobilenetv2} & 8 & 31.44 & 11.98 & 10.78 & 15.79 & 14.05 & 10.64 & 77.04 & 20.67 & 11.64 & {\ul \textbf{91.40}} & 13.89 & 10.34 \\
 & 9 & 58.51 & 26.74 & 15.41 & 12.07 & 50.69 & 14.07 & 82.77 & 82.41 & 13.67 & {\ul \textbf{92.80}} & 68.11 & 12.17 \\ \cline{2-14} 
 & Baseline & 92.69 & 93.49 & 91.45 & 93.60 & 93.74 & 90.50 & 92.97 & 93.58 & 91.84 & 93.68 & 94.03 & 91.58 \\ \hline
\end{tabular}
\end{table*}

\section{Improving low-precision PANN's accuracy}

\par This section introduces our methods for improving the ``sturdiness'' of neural networks and thus improving low-precision PANN's accuracy.


\subsection{Negative Inputs Leakage Method}
We propose a solution to enhance ``sturdiness'' based on Theorem \ref{th:pos_neg_loss} that approximation errors on the negative inputs of ReLU increase more loss than those on positive inputs. Intuitively, because some gradients stop updating due to negative inputs to ReLU having zero derivative, we can leak the information of ReLU's negative inputs to the next layer and can guide the training to approach the stationary point of PANN, thus increasing the sturdiness.

\par Formally, for a neural network with layered architectures, the input layer is numbered $0$, the output layer is numbered $H$, and the hidden layers $h_l$ are numbered $1$, $\dots, H-1$. We denote $z_{l}$ as the input at layer $l$ before the activation function and $z_{H}$ is the output of the neural network, $x_l$ as the vector after application of the activation function (to $z_{l-1}$) and $x_0=x$ is the neural network input. For normal models, $z_{l} = h_l(x_l)$ and $x_{l+1}=\operatorname{ReLU}(z_{l})$. For our method, we use $\overline{\mathbb{F}}(x)$ to represent the $\mathbb{F}(x)$ replacing all $x_{l+1}$ with $\overline{x}_{l+1}=\operatorname{z_l} + \epsilon_{in,l}$, use $\beta'$ to represent the amplification due to the multiplication of errors with intermediate values. We are looking for a neural network that satisfies the following:
\begin{equation} \label{PGDAT}
        \inf_\theta  \mathbb{E}_{\mathbb{P}} [L(\overline{\mathbb{F}}(x),y)], \quad \forall ||\epsilon_{in,l}|| < 2^{-\beta+\beta'}
\end{equation}
\par When trying to implement this method, there are two considerations. First, approximation errors affect every input to ReLU. Second, approximation errors are generated in almost every layer. The wide-range perturbations after accumulated and amplified result in a much larger noise amplitude, which will affect the training and reduce the accuracy of the backbone model (baseline), which can limit the upper bound of PANN accuracy.

Therefore, we consider restricting perturbations only to specific positions by taking into account two specific properties.
First, as implied in Theorem \ref{th:pos_neg_loss}, approximation errors on ReLU's negative inputs are more harmful. 
Second, approximation errors are usually the polynomial of ReLU inputs $z$ such that the error amplitude is proportional to the amplitude of $z$. 
The two properties suggest that adding perturbations to negative values with large magnitudes is the most effective and has minimal effect on the backbone model's performance. 
\par With the reasons above, we design a \textbf{noise generator for negative values with large magnitude (NGNV)} in activation function inputs during the training phase. For the smallest $r$ (e.g.: 0.3) percent negative values in the ReLU inputs (the negative value with the largest magnitudes), we generate Gaussian noise $\epsilon_{gau}$ to simulate $\epsilon_{sgn}$ and multiply with these values. The products are then multiplied by a parameter $\lambda$ (e.g.: 0.05) and added to ReLU inputs. In other words, for a ReLU input $z$, we generate intra-model perturbations $\delta_{int}$ as:
\begin{equation} \label{eq:gaussian_gen}
    \delta_{int} = \lambda \epsilon_{gau} z, \quad \epsilon_{gau}  \leftarrow \mathcal{N}(0,1)
\end{equation}
\par Considering that the approximation error $\epsilon_{relu}$ has boundaries. We can extract signs from $\epsilon_g$ and simulate the worst case by setting a fixed $\lambda$. The process is:
\begin{equation} \label{eq:fixnoise_gen}
    \delta_{int} = \lambda \operatorname{sgn}(\epsilon_{gau}) z, \quad \epsilon_{gau}  \leftarrow \mathcal{N}(0,1)
\end{equation}

\subsection{Reduced Use of Weight Regularization}
As discussed in Theorem~\ref{th:loss_bound_wd_t}, weight regularization increases PANNs' test loss and reduces accuracy as training processes. One potential solution is to use minimal weight regularization combined with early stopping. However, it will lead to poor generalization, which can reduce the accuracy of backbone models (baseline) and ultimately limit PANN's accuracy by decreasing its upper bound.
\par To address this, we choose Mixup as a remedy method. Mixup~\cite{zhang2018mixup} introduces globally linear behavior in-between data manifolds. It can reduce unexpected behavior of neural networks and enhance generalization and robustness by expanding the input set. Our experiments show that Mixup can somewhat compensate for the reduced accuracy due to poor regularization. The test results based on minimal regularization and Mixup are presented in Section 5.2, which shows good performance on some model structures and precision. However, it is important to note that this approach is only effective when regularization is not crucial. For deeper models, the problem of low backbone accuracy under poor regularization is still challenging. Besides, Mixup may decrease PANN's accuracy for certain models, such as Shufflenetv2. Thus, we regard use of low weight regularization and Mixup as a trade-off.

\section{Experiments}

We conduct experiments to demonstrate the effectiveness of our solutions in improving sturdiness and the performance of PANN. The main results are shown in this section and additional experiment results can be found in Appendix~\ref{app:exp2}.

\subsection{Test on Normal Structure without Fine-tuning}\label{sec:ExperimentsNormal}

\begin{table}[!htb]
\caption{Accuracy of PANN with error bound $2^{-\beta}$, and the backbone (baseline) on CIFAR-100 (ResNet-32) and TinyImagenet (ResNet-18).}
\label{tb:results_2}
\centering
\resizebox{\columnwidth}{!}{
\begin{tabular}{c|c|ll|ll|ll}
\hline
\multicolumn{1}{l|}{} & \multicolumn{1}{l|}{} & \multicolumn{2}{c|}{Vanilla} & \multicolumn{2}{c|}{Mixup} & \multicolumn{2}{c}{Mixup+NGNV} \\ \hline
\multicolumn{1}{l|}{} & \diagbox[width=1.5cm]{$\beta$}{wd} & \multicolumn{1}{c}{0} & \multicolumn{1}{c|}{1e-4} & \multicolumn{1}{c}{0} & \multicolumn{1}{c|}{1e-4} & \multicolumn{1}{c}{0} & \multicolumn{1}{c}{1e-4} \\ \hline
\multirow{3}{*}{C100} & 8 & 39.52 & 3.85 & 59.71 & 8.87 & {\ul \textbf{63.49}} & 29.10 \\
 & 9 & 61.08 & 44.73 & 66.35 & 55.61 & {\ul \textbf{67.67}} & 63.50 \\
  \cline{2-8} 
 & Baseline & 65.08 & 68.88 & 67.92 & 70.96 & 68.85 & 70.32 \\ \hline
\multirow{3}{*}{\begin{tabular}[c]{@{}c@{}}Tiny-\\      Imagenet\end{tabular}} & 8 & 18.34 & 0.73 & 20.80 & 0.52 & {\ul \textbf{43.85}} & 1.88 \\
 & 9 & 42.86 & 4.54 & 45.05 & 0.85 & {\ul \textbf{55.60}} & 23.41 \\ \cline{2-8} 
 & Baseline & 56.62 & 59.62 & 58.78 & 62.01 & 59.42 & 61.30 \\ \hline
\end{tabular}
}
\end{table}

\par \textbf{Evaluation Setup: }We conduct tests on PANN with ResNet \cite{he2016deep}, DLA \cite{yu2018deep}, MobilenetV2 \cite{sandler2019mobilenetv2}, and ShufflenetV2 \cite{Ma_2018_ECCV} models for classification tasks on CIFAR-10. 
We also test ResNet on CIFAR-100 and Tiny Imagenet dataset \cite{krizhevsky2009learning,Le2015TinyIV}. Training on CIFAR dataset takes 200 epochs, and on Tiny Imagenet takes 150 epochs. All learning rates begin with 0.1, and were multiplied by 0.1 on epochs 100 and 150 on CIFAR dataset, multiplied by 0.1 on epochs 50 and 100 on Tiny Imagenet. The momentum is 0.9. We test weight decay on 0, 0.0001, and 0.0005. Note that the ResNet-18 models removed the Maxpool layer to suit small images. $\lambda$ for Mixup are selected from distribution $\operatorname{Beta}(0.5,0.5)$. 
The approximation intervals are $[-50,50]$ for CIFAR and $[-100,100]$ for Tiny Imagenet, except the case values exceed the approximation interval. The average accuracy of the backbone models and PANN with precision $2^{-8}$ and $2^{-9}$ on CIFAR-10 dataset are presented in Table \ref{tb:results}. The experiments on CIFAR-100 (C100) and Tiny Imagenet are shown in Table \ref{tb:results_2}. The highest accuracy for each model and approximation precision are highlighted in bold.

\par \textbf{Effects of Weight Regularization:} 
From Table \ref{tb:results} and Table \ref{tb:results_2}, we observe that weight regularization can severely destroy the model's resistance to approximation errors in various models and precisions. For example, PANN on models trained with weight decay 0 outperforms those with weight decay 5e-4 in nearly all vanilla models, suggesting minimal weight decay. Note that this discipline doesn't always hold. Poor regularization caused by no regularization can limit backbone models' accuracy, thus reducing PANN's accuracy by limiting the upper bound. For example, ResNet-20 models with weight decay 1e-4 can outperform those with zero weight decay in precision $2^{-9}$.

\par \textbf{Evaluation of Minimal Regularization with Mixup:} Table \ref{tb:results} and Table \ref{tb:results_2} shows that Mixup can effectively compensate for the lack of generalization and reduced backbone accuracy caused by poor regularization. Therefore, models can benefit from strong resistance to approximation errors which come with minimal weight regularization. This method can improve PANN performance in many cases. For example, Mixup can outperform vanilla models on precision $2^{-9}$ with zero weight decay in ResNet-20 and DLA-34. However, this advantage may not always hold. In some cases, Mixup might be harmful to PANN (e.g.: DLA-34, weight decay 0, precision $2^{-8}$). 

\par \textbf{Evaluation of Negative Inputs Leakage Solution:} Table \ref{tb:results} and Table \ref{tb:results_2} prove that NGNV (or NGNV with Mixup), can achieve the best accuracy in all precisions (bolded). It can also improve PANN accuracy significantly in almost all cases, even if the weight decay is large. For instance, on Shufflenetv2 and Mobilenetv2, our models can improve PANN accuracy by up to 60\% on weight decay 0 and precision $2^{-8}$, by over 40\% on weight decay 1e-4 and precision $2^{-9}$.

\par \textbf{Cross-work comparison:} Table \ref{tb:cross_com} and \ref{tb:cross_com2} provide a cross-work comparison for different schemes. Note that due to the large amount of fine-tuning and structural adjustment involved in the different schemes, it's difficult to assess the magnitude of errors in them. So we only present the highest accuracy obtained. And in this experiment, our scheme uses the same approximation precision as MPCNN~\cite{lee2022low}. The results indicate that though our models are trained on the original backbone structure without fine-tuning with approximation functions, they can achieve nearly the same accuracy as the baseline when loaded in PANN. In contrast, model modification and fine-tuning are necessary for other high-accuracy schemes (e.g., reducing the number of layers containing activation functions). \footnote{Note that our method has the potential to be integrated with these fine-tuning schemes and enhance their performance, but due to the limited open-source material available in this specific area, we have only conducted tests using a subset of the available schemes.}


\begin{table}[!htb]
\caption{Best accuracy of PANN on CIFAR-10 dataset. ``Fine-tuned" means if the model is fine-tuned on the PANN. ``Self-defined" means the model structure is designed for PANN specifically and cannot be identified as common-used structures. }
\label{tb:cross_com}
\centering
\resizebox{\columnwidth}{!}{
\begin{tabular}{c|l|c|c|c}
Model & \multicolumn{1}{c|}{Schemes} & \multicolumn{1}{l|}{Fine-tuned} & \multicolumn{1}{l|}{Accuracy} & \multicolumn{1}{l}{Baseline} \\ \hline
\multirow{3}{*}{ResNet20}  & AutoFHE~\cite{cryptoeprint:2023/162} & Yes & 92.66 & 92.89 \\
& MPCNN~\cite{lee2022low} & No & 91.39 & 91.52 \\
 & Ours & No & 92.05 & 92.02 \\ \hline
\multirow{3}{*}{ResNet56}  & AutoFHE~\cite{cryptoeprint:2023/162} & Yes & 93.27 & 93.49 \\
& MPCNN~\cite{lee2022low} & No & 93.12 & 93.26 \\
 & Ours & No & 93.77 & 93.81 \\ \hline
\multicolumn{1}{l|}{ResNet56-Modified} & LSPCNN~\cite{baruch2023training} & Yes & 93.27 & 97.80 \\ \hline
\multirow{2}{*}{Self-defined} & HCNN~\cite{al2020towards} &  & 77.55 & \multicolumn{1}{l}{} \\
 & nGraph-HE~\cite{boemer2019ngraph} & \multicolumn{1}{l|}{} & 62.10 & \multicolumn{1}{l}{}
\end{tabular}
}
\end{table}
\begin{table}[!htb]\small
\caption{Best accuracy of PANN on CIFAR-100 dataset}
\label{tb:cross_com2}
\centering
\resizebox{\columnwidth}{!}{
\begin{tabular}{c|l|c|c|c}
Model & \multicolumn{1}{c|}{Schemes} & Fine-tuned & Accuracy & Baseline \\ \hline
\multirow{3}{*}{ResNet32} & AutoFHE~\cite{cryptoeprint:2023/162} & Yes & 70.81 & 71.34 \\
& MPCNN~\cite{lee2022low} & No & 69.50 & 69.43 \\
& Ours & No & 70.72 & 70.96
\end{tabular}
}
\end{table}

\begin{figure}[t]
        \includegraphics[width=0.9\columnwidth]{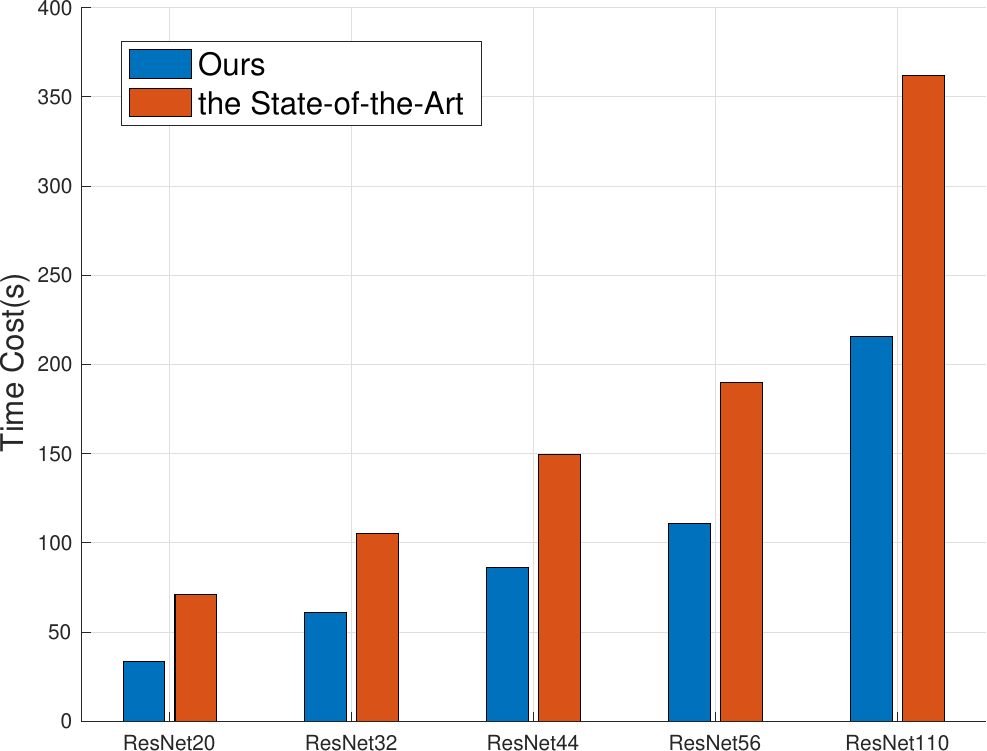}
    \caption{Time cost for PANN with our models and the non-fine-tuning SOTA~\cite{lee2021privacypreserving,  lee2022low} (to achieve the SOTA accuracy on CIFAR-10).} 
\label{fig:Introduction2}
\end{figure}
\par In a word, by combining these methods, models can exhibit strong ``sturdiness'' to approximation errors. We can significantly improve low-precision PANN's accuracy and thus reduce the time cost. 
Table \ref{tab:TimeCost} illustrates some time costs for PANN with different precision on different models. Previous schemes require at least $2^{-12}$ precision to achieve satisfactory accuracy, but our method only needs $2^{-9}$ (Kindly refer to the appendix for the selection of parameters). The time required for PANN with our models to achieve the state-of-the-art \cite{lee2021privacypreserving,  lee2022low} accuracy is given in Figure \ref{fig:Introduction2}. Our solution can reduce the PANN time cost by 40\% to 60\% compared to models provided by these schemes. This improvement makes it easier to design cryptographic schemes for various scenarios.

\subsection{Test on ReLU Replacement Schemes}\label{sec:Test_ReRe}

Our method is based on standard structures without fine-tuning (mentioned in Section \ref{sec:introoverview}), which ensures our model can benefits other schemes by improving sturdiness. For instance, our models can serve as pre-trained and teacher models (for distillation) to enhance the accuracy of ReLU Replacement schemes SNL~\cite{cho2022selective} and AutoReP~\cite{peng2023autorep}. 

\begin{table}[!htb]\small
\caption{The accuracy (\%) of SNL, AutoReP, and the two schemes with our pre-trained model. (CIFAR-100, ResNet-18, baseline: 77.9\% for ours, 77.8\% for SNL and AutoReP)}
\label{tb:relurep}
\centering
\begin{tabular}{c|c|c}
\multicolumn{1}{l|}{} & ReLUs(K) & Accuracy(\%) \\ \hline
\multirow{2}{*}{Sphynx~\cite{cho2022selective}} & 51.2 & 69.57 \\
 & 25.6 & 66.13 \\ \hline
\multirow{2}{*}{DeepReduce~\cite{jha2021deepreduce}} & 49.2 & 69.5 \\
 & 12.3 & 64.97 \\ \hline
\multirow{3}{*}{SNL~\cite{cho2022selective}} & 49.9 & 73.75 \\
 & 24.9 & 70.05 \\
 & 15.0 & 67.17 \\ \hline
\multirow{3}{*}{SNL+Ours} & 49.9 & 74.46 \\
 & 24.9 & 71.18 \\
 & 15.0 & 68.62 \\ \hline
\multirow{3}{*}{AutoReP~\cite{peng2023autorep}} & 50.0 & 75.48 \\
 & 12.9 & 74.92 \\
 & 6.0 & 73.79 \\ \hline
\multirow{3}{*}{AutoReP+Ours} & 50.0 & 78.27 \\
 & 12.9 & 77.91 \\
 & 6.0 & 77.81
\end{tabular}
\end{table}

\par Table \ref{tb:relurep} presents the performance of AutoReP and SNL with (or without) our pre-trained ResNet-18 model on CIFAR-100 dataset. We also give the accuracy of Sphynx~\cite{cho2022selective}, DeepReduce~\cite{jha2021deepreduce} for comparison. Our models are trained on the backbone structure with weight decay 1e-4 for 200 epochs. Learning rates begin with 0.1, multiplied by 0.1 on epochs 100 and 150. The table shows that our models can enhance both the accuracy of both SNL and AutoRep. For example, we can improve the accuracy of AutoRep with 50K, 12.9K, and 6K ReLUs to near the upper bound (baseline).

\begin{table}[!htb]\small
\caption{The accuracy (\%) of ReLU replacement schemes SNL, AutoReP with models pre-trained and fine-tuned with different weight decay (wd) (CIFAR-100, ResNet-18)}
\label{tb:relurep_wd}
\centering
\begin{tabular}{c|ccc|c}
wd & 0 & 1e-4 & 5e-4 & ReLUs(K) \\ \hline
\multirow{2}{*}{SNL} & 70.6 & 73.3 & 70.38 & 49.90 \\
 & 68.48 & 71.17 & 68.33 & 24.90 \\ \hline
\multirow{2}{*}{AutoRep} & 72.95 & 73.44 & 75.27 & 50.00 \\
 & 69.43 & 73.32 & 73.97 & 12.90 \\ \hline
baseline & 71.81 & 74.5 & 77.66 & 
\end{tabular}
\end{table}
\par We also give the accuracy of SNL and AutoReP with models pre-trained and fine-tuned with different weight decay in Table \ref{tb:relurep_wd}. It shows that approximations with 1-degree (in SNL) and 2-degree (in AutoRep) polynomials also comply with Theorem~\ref{th:loss_bound_wd_t}, demonstrating that weight decay enlarges the loss increment due to approximation errors. Note that this improvement is achieved without using our training methods in fine-tuning, which means it can be further improved.



\subsection{Test on Fixed-Point Truncation-based ReLU}
\begin{table}[!htb]\small
\caption{The accuracy (\%) of fix-point truncation-based ReLU computation (CIFAR-100, ResNet-18)}
\label{tb:fix_mpc}
\centering
\begin{tabular}{c|cc|cc}
 & \multicolumn{2}{c|}{Vanilla} & \multicolumn{2}{c}{Ours} \\ \hline
wd & 1e-4 & 5e-4 & 1e-4 & 5e-4 \\ \hline
baseline & 74.80 & 77.82 & 78.63 & 78.85 \\ \hline
lx=6 & 74.39 & 73.09 & 78.33 & 75.98 \\
lx=8 & 74.64 & 76.58 & 78.38 & 78.22 \\
lx=10 & 74.74 & 77.57 & 78.38 & 78.71 \\
lx=12 & 74.77 & 77.68 & 78.57 & 78.83 \\
lx=14 & 74.77 & 77.80 & 78.64 & 78.88 \\
lx=16 & 74.80 & 77.80 & 78.65 & 78.84
\end{tabular}
\end{table}

The fixed-point bits ($l_x$) of ReLU's inputs are positively correlated with communication overhead in certain MPC schemes due to bit-wise operations such as truncation. By choosing a small $lx$, MPC can significantly reduce communication costs at the expense of accuracy. Our scheme's increasing model sturdiness can effectively mitigate the accuracy reduction resulting from a small $l_x$. 

\par Table \ref{tb:fix_mpc} shows the accuracy on models with different ReLU inputs bits. The model training uses the same setting as Section \ref{sec:Test_ReRe}. And we apply the computation methods in~\cite{zhou2023bicoptor}. We set the first half of the bits in lx represent integers and the second half represent decimals. 

\par The table shows that this case still experiencing external losses due to weight decay. Though the large weight decay outperforms the small weight decay in some precision because of the baseline difference, the reduced accuracy in weight decay 1e-4 is still smaller than that in 5e-4. Furthermore, by applying our methods in training backbone, models can achieve the desired accuracy at much lower precision, ultimately reducing communication overhead.

\section{Conclusions}
We uncovered crucial insights into the nature of ``sturdiness’’, a notion we introduced as a network's resistance to approximation errors. 
We proposed the theorems regarding how approximation errors increase the test loss on PANN and give constraints on the relevant boundaries.

Based on these theorems, we have explored the factors that will affect the model ``sturdiness’’ and drawn two interesting conclusions. 
First, we found that approximations on ReLU’s negative inputs result in more testing loss compared to approximations on positive inputs, based on which we developed NGNV solution that can enhance ``sturdiness" with minimal effect on the backbone model’s performance
Second, we discovered that weight regularization negatively affects ``sturdiness’’, suggesting applying minimal weight regularization and exploring alternative generalization strategies. By combining the methods above, we can significantly improve low-precision PANN’s accuracy, thus reducing time costs.

\par We believe this work represents an important step for advancements in PANNs. Our findings will encourage continued research, which we believe will lead to more efficient and effective PPML.

Beyond these contributions, we also leave some open questions for subsequent research. Firstly, our theorems only target small errors (near zero). Due to the complex accumulation and amplification of noise in the neural networks, approximation errors in some precision (i.e., $\epsilon \leq 2^{-8}$) will lead to results inconsistent with our theorem. Finding a more general mathematical model to address this issue and help with analysis is important.
Secondly, similar to AT \cite{rice2020overfitting,yu2022understanding}, NGNV suffers from a form of ``overfitting". A model may perform better beyond some precision but have worse accuracy on lower precision (i.e. models trained with small noise may perform poorly on approximation with extremely large errors). How to make models perform well in every accuracy still needs further research.

\par Overall, this research has opened up a new field and made a great contribution to PPML by providing powerful tools and new directions for the future development of PPML.

\bibliographystyle{IEEEtran}
\bibliography{main}

\newpage

\onecolumn
\begin{appendices}

\section{Additional Experiment Results}\label{app:exp2}

\begin{table}[!htb]\footnotesize 
\caption{Accuracy of PANN on CIFAR-10 dataset (ResNet-20, Shufflenetv2, DLA-34, Mobilenetv2) trained with weight decay (wd) 0, 1e-4 and 5e-4. We give Top-1 accuracy of the backbone (bb) and PANN with precision $2^{-\beta}$. Training methods include Vanilla, Mixup, and robustness-based solution NGNV. The results with the highest accuracy for each model and approximation precision are bolded.}
\label{tb:results_more}
\vskip 0.15in
\centering
\resizebox{\columnwidth}{!}{
\begin{tabular}{c|c|ccc|ccc|ccc|ccc}
\hline
 &  & \multicolumn{3}{c|}{Vanilla} & \multicolumn{3}{c|}{Mixup} & \multicolumn{3}{c|}{NGNV} & \multicolumn{3}{c}{Mixup+NGNV} \\ \hline
 & \diagbox[width=1cm, height=0.51cm]{$\beta$}{wd} & 0 & 1e-4 & 5e-4 & 0 & 1e-4 & 5e-4 & 0 & 1e-4 & 5e-4 & 0 & 1e-4 & 5e-4 \\ \hline
\multirow{6}{*}{ResNet-20} & 8 & 86.84 & 65.70 & 12.07 & 87.91 & 76.95 & 11.77 & 87.77 & 75.33 & 11.41 & {\ul \textbf{88.79}} & 78.82 & 12.53 \\
 & 9 & 89.60 & 89.69 & 58.81 & 90.52 & 89.67 & 43.23 & 90.61 & 90.43 & 73.82 & {\ul \textbf{91.02}} & 90.48 & 47.67 \\
 & 10 & 90.21 & 91.24 & 88.62 & 91.12 & 91.18 & 87.66 & 91.05 & {\ul \textbf{91.47}} & 90.70 & 91.37 & 91.39 & 89.31 \\
 & 11 & 90.33 & 91.42 & 91.21 & 91.24 & 91.50 & 91.37 & 91.07 & 91.63 & {\ul \textbf{92.09}} & 91.46 & 91.55 & 91.54 \\
 & 12 & 90.36 & 91.52 & 91.92 & 91.28 & 91.56 & 91.96 & 91.07 & 91.66 & {\ul \textbf{92.40}} & 91.49 & 91.59 & 92.08 \\ \cline{2-14} 
 & bb & 90.39 & 91.56 & 92.14 & 91.30 & 91.60 & 92.17 & 91.10 & 91.69 & 92.49 & 91.49 & 91.59 & 92.16 \\ \hline
\multirow{6}{*}{Shufflenetv2} & 8 & 16.11 & 12.25 & 11.94 & 11.26 & 12.35 & 13.23 & {\ul \textbf{84.44}} & 20.96 & 10.09 & 72.54 & 31.62 & 13.20 \\
 & 9 & 58.60 & 32.22 & 32.08 & 36.35 & 24.29 & 25.41 & {\ul \textbf{87.65}} & 85.36 & 23.15 & 84.10 & 81.57 & 28.02 \\
 & 10 & 87.39 & 83.91 & 84.96 & 70.22 & 73.34 & 74.15 & 88.32 & {\ul \textbf{90.03}} & 86.67 & 87.42 & 89.30 & 83.67 \\
 & 11 & 88.30 & 89.51 & 87.24 & 85.55 & 89.32 & 86.26 & 88.58 & {\ul \textbf{90.36}} & 88.64 & 88.56 & 89.96 & 87.77 \\
 & 12 & 88.48 & 90.07 & 88.13 & 88.99 & 89.95 & 86.70 & 88.64 & {\ul \textbf{90.61}} & 89.20 & 89.03 & 90.34 & 87.74 \\ \cline{2-14} 
 & bb & 88.54 & 90.22 & 88.60 & 89.08 & 90.19 & 87.62 & 88.66 & 90.69 & 89.60 & 89.10 & 90.54 & 88.32 \\ \hline
\multirow{6}{*}{DLA-34} & 8 & 91.69 & 69.02 & 12.78 & 90.45 & 58.01 & 13.23 & 92.10 & 50.94 & 13.58 & {\ul \textbf{93.07}} & 52.80 & 13.45 \\
 & 9 & 92.28 & 88.79 & 45.67 & 93.52 & 87.00 & 36.80 & 93.64 & 91.94 & 38.73 & {\ul \textbf{94.82}} & 91.92 & 26.46 \\
 & 10 & 93.38 & 93.06 & 88.53 & 94.34 & 93.49 & 85.44 & 93.85 & 94.30 & 89.04 & {\ul \textbf{95.07}} & 95.00 & 88.62 \\
 & 11 & 93.40 & 94.10 & 92.73 & 94.96 & 94.94 & 92.57 & 93.85 & 94.80 & 94.08 & 95.08 & {\ul \textbf{95.44}} & 94.00 \\
 & 12 & 93.43 & 94.35 & 94.45 & 94.96 & 95.08 & 94.71 & 93.85 & 94.92 & 94.85 & 95.08 & {\ul \textbf{95.56}} & 95.15 \\ \cline{2-14} 
 & bb & 93.43 & 94.38 & 95.10 & 94.96 & 95.12 & 95.41 & 93.85 & 94.92 & 95.21 & 95.09 & 95.62 & 95.58 \\ \hline
\multirow{6}{*}{Mobilenetv2} & 8 & 31.44 & 11.98 & 10.78 & 15.79 & 14.05 & 10.64 & 77.04 & 20.67 & 11.64 & {\ul \textbf{91.40}} & 13.89 & 10.34 \\
 & 9 & 58.51 & 26.74 & 15.41 & 12.07 & 50.69 & 14.07 & 82.77 & 82.41 & 13.67 & {\ul \textbf{92.80}} & 68.11 & 12.17 \\
 & 10 & 92.35 & 76.02 & 82.85 & 92.38 & 85.67 & 76.35 & 92.73 & 91.93 & 85.29 & {\ul \textbf{93.40}} & 92.39 & 74.33 \\
 & 11 & 92.64 & 92.33 & 89.54 & 93.42 & 92.44 & 88.77 & 92.98 & 93.05 & 90.43 & {\ul \textbf{93.61}} & 93.52 & 90.54 \\
 & 12 & 92.68 & 93.47 & 90.90 & 93.56 & 93.49 & 90.03 & 92.99 & 93.48 & 91.42 & 93.67 & {\ul \textbf{93.95}} & 91.31 \\ \cline{2-14} 
 & bb & 92.69 & 93.49 & 91.45 & 93.60 & 93.74 & 90.50 & 92.97 & 93.58 & 91.84 & 93.68 & 94.03 & 91.58 \\ \hline
\end{tabular}
}
\end{table}

This section gives more results of the accuracy of PANN on models trained with different methods and weight decays. The experiment results on CIFAR-10 dataset are shown in Table \ref{tb:results_more}, on CIFAR-100 (C100) and TinyImagenet is shown in Table \ref{tb:results_2_more}. Note that PANN's accuracy is very close to the backbone models for high precision like $2^{-11}$ and $2^{-12}$, which reduce the advantage of small weight regularization.

\begin{table}[!htb]\small 
\caption{Accuracy and time cost of PANN with precision $2^{-\beta}$ and the backbone (bb) on CIFAR-100 (ResNet-32) and TinyImagenet (ResNet-18). The results with the highest accuracy for each model and approximation precision are bolded.}
\label{tb:results_2_more}
\vskip 0.15in
\centering
\begin{tabular}{c|c|ccc|ccc|ccc|c}
\hline
\multicolumn{1}{l|}{}         & \multicolumn{1}{l|}{} & \multicolumn{3}{c|}{Vanilla} & \multicolumn{3}{c|}{Mixup}           & \multicolumn{3}{c|}{Mixup+NGNV}                     & Time Cost \\ \hline
                            & \diagbox[width=1cm, height=0.51cm]{$\beta$}{wd} & 0        & 1e-4    & 5e-4    & 0     & 1e-4                 & 5e-4  & 0                    & 1e-4                 & 5e-4  &           \\ \hline
\multirow{6}{*}{C100}         & 8                     & 39.52    & 3.85    & 1.00    & 59.71 & 8.87                 & 1.00  & {\ul \textbf{63.49}} & 29.10                & 1.53  & 45.2s    \\
                              & 9                     & 61.08    & 44.73   & 1.67    & 66.35 & 55.61                & 6.24  & {\ul \textbf{67.67}} & 63.50                & 22.28 & 54.8s     \\
                              & 10                    & 64.24    & 66.53   & 44.97   & 67.70 & 68.83                & 54.12 & 68.62                & {\ul \textbf{68.97}} & 61.75 & 60.2s    \\
                              & 11                    & 65.00    & 68.40   & 64.61   & 67.88 & {\ul \textbf{70.42}} & 67.45 & 68.77                & 70.05                & 67.21 & 97.8s    \\
                              & 12                    & 65.18    & 68.98   & 69.49   & 67.92 & {\ul \textbf{70.72}} & 70.02 & 68.85                & 70.32                & 68.92 & 104.3s   \\ \cline{2-12} 
                              & bb                    & 65.08    & 68.88   & 70.30   & 67.92 & 70.96                & 70.39 & 68.85                & 70.32                & 69.31 & 2.5s     \\ \hline
\multirow{6}{*}{\begin{tabular}[c]{@{}c@{}}Tiny-\\      Imagenet\end{tabular}} & 8                     & 18.34    & 0.73    & 0.56    & 20.80 & 0.52                 & 0.53  & {\ul \textbf{43.85}} & 1.88                 & 0.61  & 324.5s   \\
                              & 9                     & 42.86    & 4.54    & 0.60    & 45.05 & 0.85                 & 0.59  & {\ul \textbf{55.60}} & 23.41                & 1.04  & 419.8s   \\
                              & 10                    & 54.03    & 32.15   & 4.38    & 56.47 & 15.89                & 0.89  & {\ul \textbf{58.97}} & 51.87                & 19.69 & 406.9s   \\
                              & 11                    & 55.92    & 53.38   & 28.78   & 58.30 & 51.54                & 9.76  & {\ul \textbf{59.33}} & 58.58                & 49.98 & 714.1s   \\
                              & 12                    & 56.48    & 58.65   & 56.28   & 56.41 & {\ul \textbf{60.83}} & 53.04 & 59.36                & 60.78                & 59.80 & 952.5s   \\ \cline{2-12} 
                              & bb                    & 56.62    & 59.62   & 62.17   & 58.78 & 62.01                & 63.47 & 59.42                & 61.30                & 61.77 & 24.7s    \\ \hline
\end{tabular}

\end{table}

\par We also give parameters $r$ and $\lambda$ for our NGNV in Table \ref{tb:results_par} to help the reader reproduce our results.

\begin{table}[!htb]\footnotesize 
\caption{Best parameters ($r$, $\lambda$) for NGNV on different weight decay and methods}
\label{tb:results_par}
\vskip 0.15in
\centering
\begin{tabular}{l|l|ccc|ccc}
\hline
 & \multicolumn{1}{c|}{} & \multicolumn{3}{c|}{NGNV} & \multicolumn{3}{c}{Mixup+NGNV} \\ \hline
 & \multicolumn{1}{c|}{wd} & 0 & 1e-4 & 5e-4 & 0 & 1e-4 & 5e-4 \\ \hline
\multirow{4}{*}{CIFAR-10} & ResNet-20 & (0.3,0.3) & (0.1,0.3) & (0.3,0.2) & (0.5,0.1) & (0.5,0.1) & (0.5,0.05) \\
 & Shufflenetv2 & (0.5,0.05) & (0.1,0.3) & (0.6,0.1) & (0.5,0.1) & (0.3,0.05) & (0.5,0.1) \\
 & DLA-34 & (0.1,0.3) & (0.1,0.3) & (0.5,0.05) & (0.3,0.05) & (0.3,0.3) & (0.1,0.3) \\
 & Mobilenetv2 & (0.5,0.1) & (0.5,0.05) & (0.5,0.05) & (0.3,0.05) & (0.5,0.1) & (0.3,0.1) \\ \hline
CIFAR-100 & ResNet-32 & \multicolumn{3}{c|}{\diagbox[width=4.85cm,height=0.35cm]{ }{ }} & (0.1,0.3) & (0.3,0.3) & (0.1,0.3) \\ \hline
Tiny-Imagenet & ResNet-18 & \multicolumn{3}{c|}{\diagbox[width=4.85cm,height=0.35cm]{ }{ }} & (0.3,0.3) & (0.1,0.3) & (0.3,0.3) \\ \hline
\end{tabular}
\end{table}

\par The accuracy and time cost for models released by us and the state-of-the-art schemes on different precisions are presented in Table \ref{tb:compare_our_state}. (Time cost for ResNet-20 has been given in Table \ref{tab:TimeCost}.)

\begin{table}[!htb]\footnotesize 
\caption{Comparation of PANN~\cite{lee2022low} with us models and the state-of-the-art non-fine-tune scheme (Classification Accuracy on CIFAR-10)}
\label{tb:compare_our_state}
\vskip 0.15in
\centering
\begin{tabular}{c|l|ccccc|c}
\hline
\multicolumn{1}{c|}{$\beta$} &  & 8 & 9 & 10 & 11 & 12 & bb \\ \hline
 & State-of-the-Art & 31.79 & 87.57 & 90.52 & 91.22 & 91.43 & 91.51 \\
\multirow{-2}{*}{ResNet-20} & Ours & 90.49 & 91.74 & 91.88 & 92.03 & 92.05 & 92.02 \\ \hline
 & State-of-the-Art & 35.71 & 88.32 & 91.85 & 92.25 & 92.39 & 92.48 \\
 & Ours & 90.55 & 91.92 & 92.42 & 92.42 & 92.42 & 92.42 \\ \cline{2-8} 
\multirow{-3}{*}{ResNet-32} & Time cost & 45.49s & 53.54s & 60.81s & 97.88s & 105.33s & 2.87s \\ \hline
 & State-of-the-Art & 11.37 & 60.75 & 91.09 & 92.55 & 92.75 & 92.75 \\
 & Ours & 72.27 & 90.30 & 92.85 & 93.21 & 93.27 & 93.34 \\ \cline{2-8} 
\multirow{-3}{*}{ResNet-44} & Time cost & 56.68s & 68.57s & 86.43s & 133.87 & 149.66s & 2.90s \\ \hline
 & State-of-the-Art & 11.05 & 42.54 & 90.12 & 92.75 & 93.12 & 93.26 \\
 & Ours & 60.54 & 89.50 & 93.30 & 93.63 & 93.77 & 93.81 \\ \cline{2-8} 
\multirow{-3}{*}{ResNet-56} & Time cost & 79.32s & 93.98s & 110.73s & 173.72s & 189.85s & 3.42s \\ \hline
 & State-of-the-Art & 10.73 & 17.72 & 71.27 & 92.01 & 93.22 & 93.49 \\
 & Ours & 49.90 & 84.65 & 93.68 & 94.04 & 94.16 & 94.16 \\ \cline{2-8} 
\multirow{-3}{*}{ResNet110} & Time cost & 148.25s & 181.37s & 215.87s & 348.38s & 362.27s & 4.01s \\ \hline
\end{tabular}
\end{table}
\section{Proof of Theorem  \ref{th:pos_neg_loss} and \ref{th:loss_bound_wd_t}} \label{app:proof}

We first give two commonly used Theorem and Lemma
\begin{theorem} \label{th:mean_con}
[Mean Value Theorem for nondifferentiable functions.] Let $h: \mathbb{R} \to \mathbb{R}$ be a convex function and let $a<b$. Then there exists $c\in(a,b)$ such that $\frac{h(b)-h(a)}{b-a} \in \partial h(c)$.
\end{theorem}

\begin{lemma}\label{lemma:lr_deriv}
Let $h: \mathbb{R} \to \mathbb{R}$ be a convex function and let $\overline{z} \in \mathbb{R}$. Then $h$ has left derivative and right derivative at $\overline{z}$. Moreover, 
\begin{equation}\notag
    \mathop{\operatorname{sup}}_{z<\overline{z}}\phi_{\overline{z}}(z)
    =h^{'}_{-}(\overline{z})
    \leq h^{'}_{+}(\overline{z})
    =\mathop{\operatorname{inf}}_{z>\overline{z}}\phi_{\overline{z}}(z)
\end{equation}
where $\phi_{\overline{z}}(z)$ is defined as $\phi_{\overline{z}}(z) = \frac{h(z) -h(\overline{z})}{z-\overline{z}}$.
\end{lemma}

\begin{lemma}[Increased loss for $\overline{z}<0$]
    Let function $h: \mathbb{R} \to \mathbb{R}$ be a convex function, and be differentiable in $(-\infty,0)\cup(0,\infty)$. Let $\overline{z} \in \mathbb{R}$, $h'_{+}(ReLU(\overline{z}))=0$ and $\overline{z}<0$, let $\epsilon \in \mathbb{R}$ be a small value. There exist an $\overline{\epsilon}\in (0, \epsilon)$ satisfies:
        \begin{equation}
        0 \leq\epsilon\cdot\mathop{\operatorname{sup}}_{z<\overline{\epsilon}}\phi_{\overline{\epsilon}}(z)
        \leq \Delta h(\operatorname{ReLU}(\overline{z}))
        \leq\epsilon\cdot\mathop{\operatorname{inf}}_{z>\overline{\epsilon}}\phi_{\overline{\epsilon}}(z)
        \end{equation}
    and 
    \begin{equation}
        \lim_{\epsilon \to 0} \Delta h(\operatorname{ReLU}(\overline{z})) =\epsilon\cdot\mathop{\operatorname{inf}}_{z>\overline{0}}\phi_{\overline{0}}(z)
    \end{equation}
    where $\Delta h(\operatorname{ReLU}(\overline{z}))=h(\operatorname{ReLU}(\overline{z})+\epsilon) - h(\operatorname{ReLU}(\overline{z}))$
\end{lemma}

\begin{proof}
    From $h'(ReLU(\overline{z}))=0$ we can get $h(\operatorname{ReLU}(\overline{z})+\epsilon) \geq h(\operatorname{ReLU}(\overline{z}))$. 
    From \ref{th:mean_con} we know there exist a $\overline{\epsilon}\in (0, \epsilon)$ satisfies:
    \begin{equation}
        h(\operatorname{ReLU}(\overline{z})+\epsilon) - h(\operatorname{ReLU}(\overline{z}))\in [\epsilon\cdot\mathop{\operatorname{sup}}_{z<\overline{\epsilon}}\phi_{\overline{\epsilon}}(z), \epsilon\cdot\mathop{\operatorname{inf}}_{z>\overline{\epsilon}}\phi_{\overline{0}}(z)]
    \end{equation}

    From Theorem \ref{th:mean_con}, it's easy to get that for $z1<z2 \in \mathbb{R}$, we have:
    \begin{equation}
        \mathop{\operatorname{sup}}_{z<z_1}\phi_{z_1}(z)
        \leq \mathop{\operatorname{inf}}_{z>z_1}\phi_{z_1}(z)
        \leq \mathop{\operatorname{sup}}_{z<z_2}\phi_{z_2}(z)
        \leq \mathop{\operatorname{inf}}_{z>z_2}\phi_{z_2}(z)
    \end{equation}
    so we have:
    \begin{equation} \label{eq:zleq0_ins}
        0
        \leq \mathop{\operatorname{inf}}_{z>0}\phi_{0}(z)
        \leq \mathop{\operatorname{sup}}_{z<\overline{\epsilon}}\phi_{\overline{\epsilon}}(z)\leq \mathop{\operatorname{inf}}_{z>\overline{\epsilon}}\phi_{\overline{\epsilon}}(z)
    \end{equation}
    and:
    \begin{equation}
        \lim_{\epsilon \to 0} \mathop{\operatorname{sup}}_{z<\overline{\epsilon}}\phi_{\overline{\epsilon}}(z) =\lim_{\epsilon \to 0} \mathop{\operatorname{inf}}_{z>\overline{\epsilon}}\phi_{\overline{\epsilon}}(z) =\mathop{\operatorname{inf}}_{z>0}\phi_{0}(z)
    \end{equation}
    so we have:

        \begin{equation}
        h(\operatorname{ReLU}(\overline{z})+\epsilon) - h(\operatorname{ReLU}(\overline{z}))\in [\mathop{\operatorname{inf}}_{z>0}\phi_{0}(z), \mathop{\operatorname{inf}}_{z>0}\phi_{0}(z)]
    \end{equation}
    which can get:
    $\Delta h(\operatorname{ReLU}(\overline{z})) = \mathop{\operatorname{inf}}_{z>0}\phi_{0}(z)$
    
    The proof is complete.
\end{proof}
\begin{lemma}[Increased loss for $\overline{z}>0$]
    Let function $h: \mathbb{R} \to \mathbb{R}$ be a convex function, and be differentiable in $(-\infty,0)\cup(0,\infty)$. Let $\overline{z} \in \mathbb{R}$, $h'(ReLU(\overline{z}))=0$ and $\overline{z}>0$, let $\epsilon \in \mathbb{R}$ be a small value. We have:
        \begin{equation}
        \lim_{\epsilon \to 0} \Delta h(\operatorname{ReLU}(\overline{z})) = 0
    \end{equation}
    where $\Delta h(\operatorname{ReLU}(\overline{z}))=h(\operatorname{ReLU}(\overline{z})+\epsilon) - h(\operatorname{ReLU}(\overline{z}))$
\end{lemma}

\begin{theorem}
    For two convex function $h_1: \mathbb{R} \to \mathbb{R}$ and $h_2: \mathbb{R} \to \mathbb{R}$ differentiable in $(-\infty,0)\cup(0,\infty)$. Let $\overline{z}_1 \in \mathbb{R}$ and $\overline{z}_2 \in \mathbb{R}$ and $\overline{z}_1<0, \overline{z}_2>0$, such that $h_1'(\overline{z}_1)=h_2'(\overline{z}_2)=0$. Let $\epsilon \in \mathbb{R}$ be a small value, we have:
    \begin{equation}
        \lim_{\epsilon \to 0} \Delta h_1 - \Delta h_2 = \epsilon \cdot \mathop{\operatorname{inf}}_{z>\overline{0},h_1}\phi_{\overline{0}}(z) 
    \end{equation}
\end{theorem}

\begin{lemma} \label{deltah}
    Let function $h: \mathbb{R} \to \mathbb{R}$ be a convex function, and be differentiable in $(-\infty,0)\cup(0,\infty)$. Let $\overline{z} \in \mathbb{R}$ and $h'(\operatorname{ReLU}(\overline{z}))=0$, let $\epsilon \in \mathbb{R}$ be a small value, we have:
        \begin{equation}
        \epsilon\cdot\mathop{\operatorname{inf}}_{z>\overline{z}}\phi_{\overline{z}}(z) \leq \Delta h(\operatorname{ReLU}(\overline{z}))
        \end{equation}
    where $\Delta f(\operatorname{ReLU}(\overline{z}))=f(\operatorname{ReLU)(\overline{z}}+\epsilon) - f(\operatorname{ReLU}(\overline{z}))$
\end{lemma}

\begin{proof}
    For $\overline{z}>0$, $\mathop{\operatorname{sup}}_{z<\overline{z}}\phi_{\overline{z}}(z)=\mathop{\operatorname{inf}}_{z>\overline{z}}\phi_{\overline{z}}(z)=h'(\overline{z})$, so we have:
    \begin{equation}
        \epsilon\mathop{\operatorname{inf}}_{z>\overline{z}}\phi_{\overline{z}}(z) \leq \Delta h(\operatorname{ReLU}(\overline{z}))
    \end{equation}
    For $\overline{z} <0$, from Equation \ref{eq:zleq0_ins}, we have:
    \begin{equation}
        \mathop{\operatorname{inf}}_{z>\overline{z}}\phi_{\overline{z}}(z)
                \leq0
        \leq \mathop{\operatorname{inf}}_{z>0}\phi_{0}(z)
        \leq \mathop{\operatorname{sup}}_{z<\overline{\epsilon}}\phi_{\overline{\epsilon}}(z)\leq \mathop{\operatorname{inf}}_{z>\overline{\epsilon}}\phi_{\overline{\epsilon}}(z)
    \end{equation}
    so we can also get:
    \begin{equation}
        \epsilon\mathop{\operatorname{inf}}_{z>\overline{z}}\phi_{\overline{z}}(z) \leq \Delta h(\operatorname{ReLU}(\overline{z}))
    \end{equation}
        
    The proof is complete.
\end{proof}
And from~\cite{xie2024overlooked} we have the theorem:
\begin{theorem} \label{eq:wd_incnorm}
    For an $\mathcal{L}$-smooth function $L(W)$, assume $L$ is lower-bounded as $L(W) \geq L^{\star} $, $L(W, X)$ is the loss over one minibatch $X$, $\sigma^{2}$ is the gradient noise variance, $\mathbb{E}[\nabla L(W, X) - \nabla L(W) ] = 0$, $\mathbb{E}[\| \nabla L(W, X) - \nabla L(W) \|^{2} ] \leq \delta^{2} $, and $\| \nabla L(W) \| \leq G$ for any $W$. If $\eta \leq \frac{C}{\sqrt{t+1}} $:
\begin{align} \label{eq:gradient_bound}
 \min_{k=0,\ldots, t} \mathbb{E}[\| \nabla f(W_{k})\|^{2} ]  \leq \frac{1}{\sqrt{t+1}}  \left[C_1 + C_2\right],
\end{align}
where
\begin{align}
C_1 &=\frac{L(W_{0}) + \frac{\lambda}{2} \|W_{0}\|^{2}  - L^{\star} }{C},\\
C_2 &= C (\mathcal{L} + \lambda) ((G+  \sup(\lambda\|W\|) )^{2} + \sigma^{2}),
\end{align}
\end{theorem}

\begin{theorem}

    Let function $L: \mathbb{R} \to \mathbb{R}$ be a $\mathcal{L}$-smooth function and $h$ be a convex function with $h'(\operatorname{ReLU}(\overline{z}))=0$.
    Assume $h$ is differentiable in $(-\infty,0)\cup(0,\infty)$. Let $\epsilon \in \mathbb{R}$ be a small value. $\mathbb{E}[ \nabla L(\bfit{W}, X) - \nabla L(\bfit{W})  ] = 0$, $\mathbb{E}[\| \nabla L(\bfit{W}, X) - \nabla L(\bfit{W}) \|^{2} ] \leq \sigma^{2}$, $\| \nabla l(\bfit{W}) \| \leq G$ for any $\bfit{W}$, and $\eta \leq \frac{C}{\sqrt{t+1}} $.
    If the model is trained $t'+1$ epochs after it has reached the stationary point $\bfit{W}^{*}$ where $L(\bfit{W}^{*}, X)=L^{\star}$, the increased loss $\Delta L=h(\operatorname{ReLU}(\overline{z}+\epsilon)) - h(\operatorname{ReLU}(\overline{z}))$ caused by error $\epsilon$ is bounded by:
        \begin{equation}
        ||\epsilon||^2 \cdot\mathbb{E}[|| \mathop{\operatorname{inf}}_{z>\overline{z}}\phi_{\overline{z}}(z)||^2] \leq \mathbb{E}[||\Delta h||^2]
        \end{equation}
        where 
        \begin{equation}
            \operatorname{sup} \mathbb{E}[||\operatorname{inf}_{z>\overline{z}}\phi_{\overline{z}}(z)||^2] \leq ||h'_{+,t}(\overline{z})||^2 + \frac{t'(C_1+C_2)}{d\sqrt{t+t'+1}}
        \end{equation}
        and:
        \begin{align}
            C_1 &=\frac{L(\bfit{W}_{0}) - L^{\star} }{C} \\
            C_2 &= C (\mathcal{L} + \lambda) ((G+  \sup(\lambda\|W\|) )^{2} + \sigma^{2}),
        \end{align}
\end{theorem}

\begin{proof}
    From lemma \ref{deltah}, we have:
    \begin{equation}
                \epsilon\mathop{\operatorname{inf}}_{z>\overline{z}}\phi_{\overline{z}}(z) \leq \Delta h(\operatorname{ReLU}(\overline{z}))
    \end{equation}
    It's easy to get 
    \begin{equation}
        \operatorname{sup}\mathbb{E}[||\mathop{\operatorname{inf}}_{z>\overline{z}}\phi_{\overline{z}}(z)||^2] = ||h'_{+,t}(\overline{z})||^2 + \frac{1}{d}\mathbb{E}[\sum_{i=t}^{t+t'}||\nabla f (\operatorname{ReLU}(z))||^2]
    \end{equation}
    From theorem \ref{eq:wd_incnorm} we have:
    \begin{equation}
        \sup \mathbb{E}[\sum\limits_{i=t}^{t+t'+1}||\nabla f (\operatorname{ReLU}(z))||^2]=  \mathop{\operatorname{\min}}_{k=0,\ldots, t+t'+1} (t'+1)\mathbb{E}[\| \nabla f(W_{k})x\|^{2}] \leq \frac{(t'+1) C_{0}}{\sqrt{t+t'+1}}
    \end{equation}
    such that we have:
    \begin{equation}
        \operatorname{sup} \mathbb{E}[||\mathop{\operatorname{inf}}_{z>\overline{z}}\phi_{\overline{z}}(z)||^2] = ||h'_{+,t}(\overline{z})||^2 + \frac{1}{d}\sup \mathbb{E}[\sum\limits_{i=t}^{t+t'+1}||\nabla f (\operatorname{ReLU}(z))||^2] \leq ||f'_{+,t}(\overline{z})||^2 + \frac{(t'+1) C_{0}}{d\sqrt{t+t'+1}}
    \end{equation}
            
    The proof is complete.
\end{proof}

\section{Proof of theorem \ref{eq:wd_incnorm}} \label{app:proof_fromotherpaper}

Here we borrow the proof from~\cite{xie2024overlooked} for reader's reference:

\begin{lemma}[Convergence of SGD, a specialized case of Theorem 1 in \cite{yan2018unified} with $\beta =s =0$]
 \label{pr:sgdconverge}
Assume that $L(\theta)$ is an $\mathcal{L}$-smooth function\footnote{It means that $\|\nabla L(\theta_{a} ) - \nabla L(\theta_{b} ) \| \leq \mathcal{L} \| \theta_{a} - \theta_{b}\|$ holds for any $\theta_{a}$ and $\theta_{b}$.}, $L$ is lower bounded as $L(\theta) \geq L^{\star} $, $\mathbb{E}[ \nabla L(\theta, X) - \nabla L(\theta)  ] = 0$, $\mathbb{E}[\| \nabla L(\theta, X) - \nabla L(\theta) \|^{2} ] \leq \delta^{2} $, $\| \nabla L(\theta) \| \leq G$ for any $\theta$. Let SGD optimize $L$ for $t+1$ iterations. If $\eta \leq \frac{C}{\sqrt{t+1}} $, we have
\begin{align}
 & \min_{k=0,\ldots, t} \mathbb{E}[\| \nabla L(\theta_{k})\|^{2} ]  \leq \frac{C_{0}}{\sqrt{t+1}},
\end{align}
where $C_{0} = \left[\frac{L(\theta_{0}) - L^{\star} }{C} + C \mathcal{L} (G^{2} + \sigma^{2})\right]$.
\end{lemma}

\begin{proof}
Given the conditions of $L(\theta)$ in Lemma \ref{pr:sgdconverge}, we may obtain the resulted conditions of $f(\theta) = L(\theta) + \frac{\lambda}{2}\|\theta\|^{2}$.

As $L(\theta)$ is an $\mathcal{L}$-smooth function, we have
\begin{align}
\|\nabla f(\theta_{a} ) - \nabla f(\theta_{b} ) \| = \|\nabla L(\theta_{a} ) - \nabla L(\theta_{b} ) + \lambda (\theta_{a} - \theta_{b}) \| \leq (\mathcal{L} + \lambda) \| \theta_{a} - \theta_{b}\|
\end{align}
holds for any $\theta_{a}$ and $\theta_{b}$. It shows that $f(\theta)$ is an $(\mathcal{L}+\lambda)$-smooth function.

As $L$ is lower bounded as $L(\theta) \geq L^{\star}$, we have 
\begin{align}
f^{\star} \geq L^{\star}.
\end{align}

As $\mathbb{E}[ \nabla L(\theta, X) - \nabla L(\theta)  ] = 0$, we have
\begin{align}
\mathbb{E}[ \nabla f(\theta, X) - \nabla f(\theta)  ] = \mathbb{E}[ \nabla L(\theta, X) - \nabla L(\theta)  ] = 0.
\end{align}

As $\mathbb{E}[\| \nabla L(\theta, X) - \nabla L(\theta) \|^{2} ] \leq \delta^{2} $, we have
\begin{align}
\mathbb{E}[\| \nabla f(\theta, X) - \nabla f(\theta) \|^{2} ] = \mathbb{E}[\| \nabla L(\theta, X) - \nabla L(\theta) \|^{2} ] \leq \delta^{2}.
\end{align}

As $\| \nabla L(\theta) \| \leq G$, we have
\begin{align}
\| \nabla f(\theta) \| = \| \nabla L(\theta) + \lambda \theta \| \leq G + \lambda \|\theta\|_{\max},
\end{align}
where $\|\theta\|_{\max}$ is the maximum $L_{2}$ norm of any $\theta$.

Introducing the derived conditions Eq. (12) - (16) for $f$ into Lemma \ref{pr:sgdconverge}, we may treat $f$ as the objective optimized by SGD. Then we have
\begin{align}
 \min_{k=0,\ldots, t} \mathbb{E}[\| \nabla f(\theta_{k})\|^{2} ]  &\leq \frac{1}{\sqrt{t+1}}  \left[\frac{f(\theta_{0})  - f^{\star} }{C} + C (\mathcal{L} + \lambda) ((G+ \lambda\|\theta\|_{\max})^{2} + \sigma^{2})\right] \\
 &\leq \frac{1}{\sqrt{t+1}}  \left[\frac{L(\theta_{0}) + \frac{\lambda}{2} \|\theta_{0}\|^{2}  - L^{\star} }{C} + C (\mathcal{L} + \lambda) ((G+ \lambda\|\theta\|_{\max})^{2} + \sigma^{2})\right]
\end{align}

Obviously, the gradient norm upper bound in convergence analysis monotonically increases as the weight decay strength $\lambda$.

The proof is complete.

\end{proof}

\section{Adversarial Samples against PANN only}
We present a method to generate adversarial samples against PANN only to help readers better observe their differences. An intuitive way to find adversarial samples only for PANN is to look for pixels where gradients differ most. However, information contributing to the outputs also has gradient differences. As discussed in Section 3, perturbations on this information can affect both models. So, we set a small value $\epsilon_{lim}$ to limit the gradient changes on the backbone model. Besides, to make the phenomenon more apparent, we set $\epsilon_{atk}$ to preserve only large gradient differences. 
\begin{figure}[h]
  \centering
  \includegraphics[width=\columnwidth]{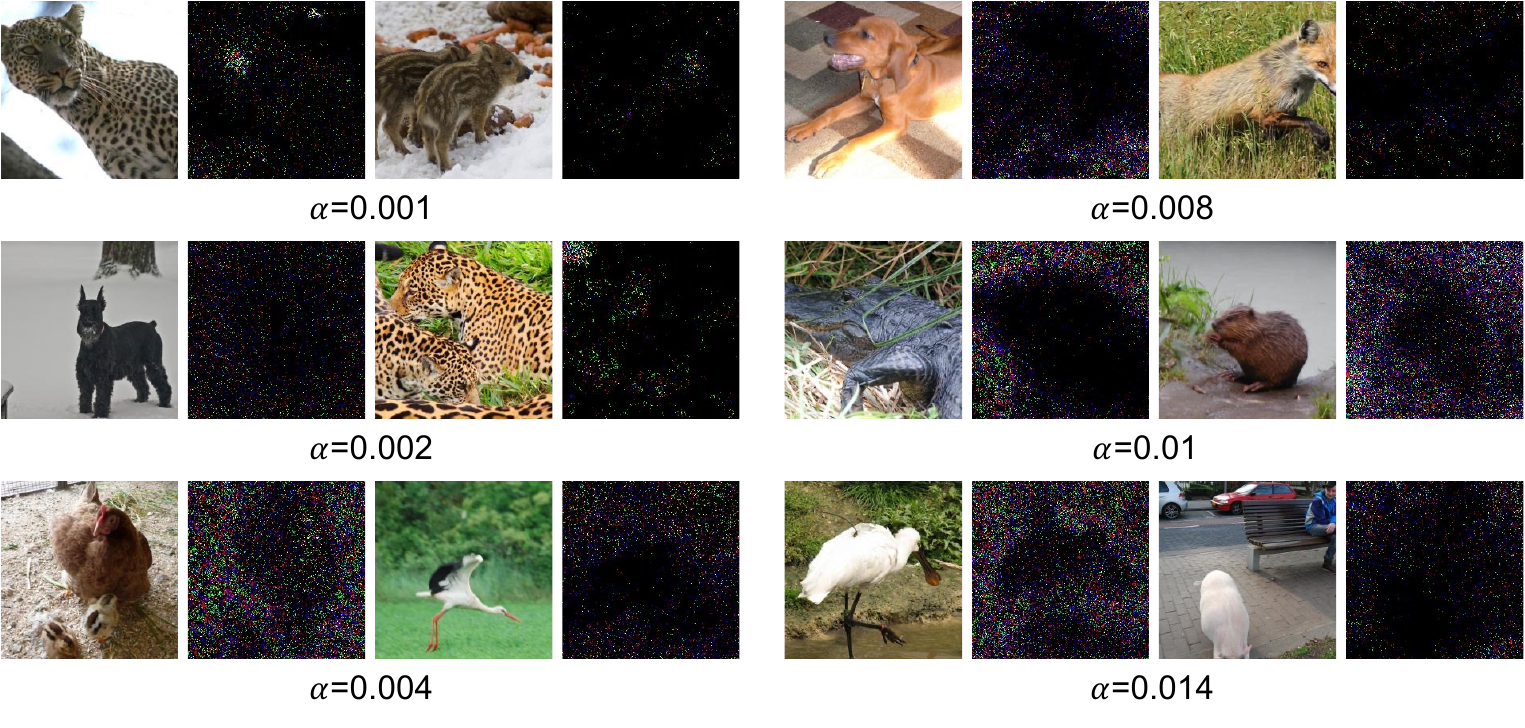}
  \caption{The benign samples and perturbations $\delta$ only against PANN.  $\delta$ can concentrate on image backgrounds (ResNet-18, Imagenet, Precision $2^{-14}$)}
  \label{fig:algorithim1}
\end{figure}
\begin{algorithm}
\caption{Evasion Attacks against PANN Only}
\label{alg:evasion_againstPANN}
\textbf{Input}: Clean Input $x$; \\
\textbf{Parameter}: Backbone model $\mathbb{F}$; PANN $\widetilde{\mathbb{F}}$; Attack Step Length $\alpha$; Limitation for Perturbations $\epsilon$ \\
\textbf{Output}: Perturbation $\delta$
\begin{algorithmic}
\STATE Initialize $\delta_0$ as all zero;
\WHILE{$\mathbb{F}(x+\delta) \neq y$ or $\widetilde{\mathbb{F}}(x+\delta) = y$}
    \IF{$\widetilde{\mathbb{F}}(x+\delta) = y$}
        \STATE $\delta = \delta + \alpha \nabla_{\delta} \mathcal{L}_{\widetilde{\mathbb{F}}}(x+\delta)$
        \STATE $\delta = Clip_{[-\epsilon,\epsilon]}(Random \ Search(x+\delta))$
        \STATE $\delta = \delta \odot Mask(
        \nabla_{\delta} (\mathcal{L}_{\widetilde{\mathbb{F}}}(x+\delta)
        -\nabla_{\delta} \mathcal{L}_{\mathbb{F}}(x+\delta)) \geq \epsilon_{atk}) $
        \STATE $\delta = \delta \odot Mask(|\nabla_{\delta} \mathcal{L}_{\mathbb{F}}(x+\delta) -
        \nabla_{\delta} \mathcal{L}_{\mathbb{F}}(x)| \leq \epsilon_{lim})$
    \ENDIF
    \IF{$\mathbb{F}(x+\delta) \neq y$}
        \STATE $Backtrack(\delta)$   
    \ENDIF
\ENDWHILE
\STATE \textbf{Return} $\delta$
\end{algorithmic}
\end{algorithm}
\par Additionally, two problems were faced when using PGD-based attacks: (1)Firm-step search can cross the adversarial samples. As shown in Figure \ref{fig:approxsgn}, approximation errors fluctuate when the input changes, so valid perturbations are in discrete intervals. It's difficult to estimate where these intervals are. Using particularly smaller steps can avoid this, but is inefficient. Therefore, we apply random searches after each fixed step. (2) Even if limits are set, the perturbations may still affect the backbone model. Once the backbone model gives a wrong output, the PGD interaction will fall into this area and can hardly escape. To avoid this, we apply backtracking once the backbone model gives the wrong results. 
The whole process is represented in Algorithm \ref{alg:evasion_againstPANN}.
\par We attack PANN with precision $2^{-14}$ and approximation interval $[-100,100]$. Some results on the Pytorch Resnet-18 pre-trained model are shown in Figure \ref{fig:algorithim1}. It can be seen that $\delta$ can concentrate on both backgrounds and objects. Especially when $\alpha$ is not too small, $\delta$ can locate mainly on backgrounds for many samples, which provides evidence for the differences in irrelevant information in the input background between PANN and backbone models.

\end{appendices}
\end{document}